\documentclass{article}

\usepackage{microtype}
\usepackage{graphicx}
\usepackage{subfigure}
\usepackage{booktabs} 

\usepackage{hyperref}

\usepackage[accepted]{icml2025}

\usepackage{amsmath}
\usepackage{amssymb}
\usepackage{mathtools}
\usepackage{amsthm}
\usepackage{multirow}
\usepackage{makecell}

\usepackage[capitalize,noabbrev]{cleveref}

\theoremstyle{plain}

\theoremstyle{definition}

\theoremstyle{remark}

\usepackage{thmtools,thm-restate}
\usepackage{cleveref}
\newcommand{\cut}[1]{}
\DeclareMathOperator*{\argmin}{arg\,min}
\DeclareMathOperator*{\argmax}{arg\,max}

\newcommand{\etal}{\textit{et al}.}

\newcommand{\ourModel}{Sharpness-Aware Trajectory Matching}
\newcommand{\ouracronym}{SATM}

\newcommand{\norm}[1]{ \| #1  \|  }

\usepackage[textsize=tiny]{todonotes}

\icmltitlerunning{}

\begin{document}

\twocolumn[
\icmltitle{Enhancing Generalization via Sharpness-Aware Trajectory Matching for \\ Dataset Condensation}



\icmlsetsymbol{equal}{*}

\begin{icmlauthorlist}
\icmlauthor{Boyan Gao}{ox}
\icmlauthor{Bo Zhao}{jiaotong}
\icmlauthor{Shreyank N Gowda}{nottingham}
\icmlauthor{Xingrun Xing}{cas}
\icmlauthor{Yibo Yang}{ox}
\icmlauthor{Timothy Hospedales}{edinburgh}
\icmlauthor{David A. Clifton}{ox}
\end{icmlauthorlist}

\icmlaffiliation{ox}{Department of Engineering Science, University of Oxford}
\icmlaffiliation{jiaotong}{School of Artificial Intelligence, Shanghai Jiao Tong University}
\icmlaffiliation{nottingham}{School of Computer Science, University of Nottingham}
\icmlaffiliation{cas}{Institute of Automation, Chinese Academy of Sciences}
\icmlaffiliation{edinburgh}{School of Informatics, University of Edinburgh}

\icmlcorrespondingauthor{Boyan Gao}{boyan.gao@eng.ox.ac.uk}

\icmlkeywords{Machine Learning, ICML}
\vskip 0.3in
]



\printAffiliationsAndNotice{}  


\begin{abstract}
Dataset condensation aims to synthesize datasets with a few representative samples that can effectively represent the original datasets. This enables efficient training and produces models with performance close to those trained on the original sets. Most existing dataset condensation methods conduct dataset learning under the bilevel (inner- and outer-loop) based optimization. However, the preceding methods perform with limited dataset generalization due to the notoriously complicated loss landscape and expensive time-space complexity of the inner-loop unrolling of bilevel optimization. These issues deteriorate when the datasets are learned via matching the trajectories of networks trained on the real and synthetic datasets with a long horizon inner-loop. To address these issues, we introduce Sharpness-Aware Trajectory Matching (SATM), which enhances the generalization capability of learned synthetic datasets by optimising the sharpness of the loss landscape and objective simultaneously. Moreover, our approach is coupled with an efficient hypergradient approximation that is mathematically well-supported and straightforward to implement along with controllable computational overhead. Empirical evaluations of SATM demonstrate its effectiveness across various applications, including in-domain benchmarks and out-of-domain settings. Moreover, its easy-to-implement properties afford flexibility, allowing it to integrate with other advanced sharpness-aware minimizers. Our code will be released. 
\end{abstract}

\section{Introduction}
The success of modern deep learning in various fields, exemplified by Segment Anything~\citep{samsegment} in computer vision and GPT~\citep{gtp_2022} in natural language processing, comes at a significant cost in terms of the enormous computational expenses associated with large-scale neural network training on massive amounts of real-world data~\cite{radford2021learning,pmlr-v202-li23q,schuhmann2022laion,li2022blip,gowda2023watt}. To reduce training and dataset storage costs, selecting the representative subset based on the specific importance criteria forms a direct solution~\citep{har2004coresets, yang2205dataset, paul2021deep, wang2022improving}. However, these methods fail to handle the cases when the samples are distinct and the information is uniformly distributed in the dataset. In contrast, Dataset Condensation (DC)~\citep{dc2021, dm, wang2018dataset, mtt, ftd} focuses on creating a small, compact version of the original dataset that retains its representative qualities. As a result, models trained on the condensed dataset perform comparably to those trained on the full dataset, significantly reducing training costs and storage requirements, and meanwhile expediting other machine learning tasks such as hyperparameter tuning, continual learning~\citep{rosasco2021distilled}, architecture search~\citep{sangermano2022sample,yu2020semantic, masarczyk2020reducing}, and privacy-preserving~\citep{shokri2015privacy, dong2022privacy}.

Given the significant practical value of condensed datasets, considerable effort has been directed toward designing innovative surrogate methods to ensure that synthetic datasets capture representative characteristics, thereby enhancing future deployments' performance~\citep{dm, dc2021, zhou2022dataset, kim2022dataset}. Bilevel optimization (BO) provides a DC paradigm learning synthetic dataset through its main optimization objective in the outer-loop constrained by training neural networks in its inner-loop. One line of the representative solutions condenses datasets by minimizing the disparity between training trajectories on synthetic and real sets, achieving notable performance~\citep{mtt}. The following studies either reduce the computational cost of inner-loop unrolling or steer the optimization process to enhance the generalization of the learned dataset to the unseen tasks. For instance, FTD~\citep{ftd} improves the performance of synthetic datasets by leveraging high-quality inner-loop expert trajectories and incorporating momentum into the outer-loop optimization via Exponential Moving Average (EMA) with extra memory overhead introduced, increasing along with the synthetic dataset budget. TESLA~\citep{TESLA} is proposed with a two-inner-loop-based algorithm to approximate the hypergradient for the dataset updates maintaining a constant memory usage. However, the outer-loop loss landscape, formed with numerous sharp regions shaped by the dynamics of the inner-loop~\citep{sharpmaml,franceschi2017forward}, is often overlooked in the dataset condensation field. The challenges inherent in such optimization settings result in the limited generalization performance of the learned datasets.

Inspired by the sharpness-aware optimizers~\citep{sam, asam, vasso}, which improves generalization by minimizing loss landscape sharpness to achieve flat convergence regions in uni-level optimization, we propose Sharpness-Aware Trajectory Matching (SATM) to DC to reduce the sharpness of the outer-loop landscape and enhance the generalization ability of the learned dataset. However, direct application is infeasible due to the tremendous computation overhead caused by the notorious two-stage gradient estimation, which typically doubles both the time and memory costs throughout the learning process. To address this issue, we propose a lightweight trajectory matching-based method composed of two computationally efficient strategies, namely truncated unrolling hypergradient and trajectory reusing, with controllable memory cost for gradient estimation. Our method improves the generalization ability of the trajectory-matching algorithm significantly on both in-domain and out-of-domain tasks with noticeable improvement margins across various applications whilst achieving efficient time and memory cost. More specifically, in terms of the computational overhead, our method surpasses one of the most efficient algorithms, TESLA~\citep{TESLA}, regarding runtime cost with comparable memory consumption. The main contributions of this work are summarised as: 
\begin{itemize}
    \item We primarily study and improve the generalization ability of dataset condensation and propose \ourModel{} by jointly minimizing the sharpness and the distance between training trajectories with a tailored loss landscape smoothing strategy.   
    \item A simple and easy-to-implement method, integrating two hypergradient approximation strategies, is proposed to handle the tremendous computational overhead introduced by sharpness minimization. We further reduce the computational redundancy by deriving a closed-form solution for the learning rate learning. For all the proposed approximation methods, we provide rigorous theoretical support by bounding the errors of the approximations and analysing the approximation error caused by hyperparameters, which shed light on meaningful hyperparameter tuning. 
    \item SATM outperforms the trajectory-matching-based competitors on various dataset condensation benchmarks under in- and out-of-domain settings. More importantly, our method demonstrates noticeable improvement margins on ImageNet-1K where most existing methods fail to condense.
\end{itemize}
\section{Related Work}
\subsection{Dataset Condensation} 
\cut{
Bilevel optimization~\citep{sinha2017review,zhang2024introduction}, nesting optimization problems as constraints for the main optimization objective, is formulated as follows: 
\begin{align}
    \min_{\phi} \, &\mathcal{L}^{outer}(\theta^*(\phi), \phi)  \\
    \textbf{s.t.}\,\, &\theta^*(\phi) = \argmin_{\theta} \, \mathcal{L}^{inner}(\theta, \phi )
\end{align}
where, $\argmin_{\theta}\mathcal{L}^{inner}(\theta, \phi)$ forms the constraint for the main optimization objective function, $\mathcal{L}^{outer}$. The learnable parameter $\phi$ in the outer-loop influences the performance of the inner-loop state, $\theta(\phi)$, while the inner-loop also depends on the current free parameter on the outer-loop. This optimization framework is widely used in various machine learning areas, including hyperparameter tuning \citep{lorraine2020optimizing, maclaurin2015gradient, mackay2019self} and meta-learning \citep{finn2017model, gao2022loss, rajeswaran2019meta, gao2021searching}.
}
Inspired by knowledge distillation~\citep{gou2021knowledge,yang2020distilling} and meta-learning driven by Bilevel optimization (BO)~\citep{lorraine2020optimizing, maclaurin2015gradient, mackay2019self, finn2017model, gao2022loss, rajeswaran2019meta, gao2021searching}, Wang~\etal~\citep{wang2018dataset} leverage BO to distill a small, compact synthetic dataset for efficient training on unseen downstream tasks. Several works expanding on this framework match gradients~\citep{dsa2021, dc2021,lee2022dataset}, features~\citep{wang2022cafe}, and distributions~\citep{dm} produced by the synthetic and real sets. RDED~\citep{rded} introduces new perspectives to the dataset distillation field by constructing synthetic images from original image crops and labelling them with a pre-trained model. Usually, the existing dataset condensation methods conduct a few iterations of inner-loop unrolling in BO to mitigate the computational cost of the nested optimization process. To avoid the same issue, Nguyen~\etal~\citep{nguyen2021dataset, nguyen2022dataset} directly estimate the convergence of the inner-loop using the Neural Tangent Kernel (NTK) to emulate the effects from the synthetic sets. However, due to the heavy computational demands of matrix inversion, the NTK-based method struggles to scale up for condensing large, complex datasets. MTT~\citep{mtt} emphasises the benefits of a long horizon inner-loop and minimizes the differences between synthetic and expert training trajectory segments with the following studies such as FTD~\citep{ftd}, TESLA~\citep{TESLA} and DATM~\citep{datm}. Nonetheless, the learned synthetic dataset often overfits the neural architecture used in the expert trajectories, resulting in limited generalization ability. In this work, we address this problem by exploring the flatness of the synthetic dataset's loss landscape.

\subsection{Flatness of the Loss Landscape and Generalization} The generalization enhanced by flat region minimums has been observed empirically and studied theoretically \citep{dinh2017sharp, keskar2016large,neyshabur2017exploring}. Motivated by this, Sharpness-aware minimizer (SAM)~\citep{sam} optimizes the objective function and sharpness simultaneously to seek the optimum lying in a flat convergence region. However, the computational overhead of SAM is double that of the conventional optimization strategy. To address this issue, ESAM~\citep{esam} randomly selects a subset of the parameters to update in each iteration. Zhuang~\etal~\citep{gsam} observes that SAM fails to identify the sharpness and mitigates this by proposing a novel sharpness proxy. To tackle the complicated loss landscape, Li and Giannakis~\citep{vasso} introduce a momentum-like strategy for sharpness approximation while ASAM~\citep{asam} automatically modify the sharpness reaching range by adapting the local loss landscape geometry. In contrast, we handle complicated multi-iteration unrolling for learning datasets in the many-shot region with both the difficulty of the sharpness approximation and the surge in computation resources.

\cut{Compared to methods that focus on the flatness of single-level training, we investigate this concept within the bilevel optimization framework, specifically exploring the benefits of flatness in outer-loop tasks. Sharp-MAML \citep{sharpmaml} pioneer in this area, enhancing the generalization ability of the learned initialisation in few-shot learning tasks through one-step inner-loop unrolling. In contrast, we handle complicated multi-iteration unrolling for learning datasets in the many-shot region where both the difficulty of approximating the sharpness and the computation resources surge. }

\section{Preliminary}
\cut{We briefly frame Sharpness-Aware minimization and Matching Training Trajectory (MTT), the main baseline model, to pave the way for developing our proposed method. 
}

\subsection{Dataset Condensation and Matching Training Trajectory}
Dataset condensation focuses on synthesizing small datasets with a few representative samples that effectively capture the essence of the original datasets. \citet{mtt} proposed to create the synthetic datasets by minimizing the distance between the training trajectory produced by the synthetic set, named synthetic trajectories, and those by the real set, termed expert trajectories, with the assumption that the datasets containing similar information generate close training trajectories, \emph{a.k.a}, matching training trajectory (MTT). A sequence of expert weight checkpoints, $\theta^E_t$, are collected during the training on the real sets in the order of iterations, $t$, to construct the expert trajectories, $\{\theta^{E}_t\}^T_{t =0}$ with $T$ denoting the total length of the trajectory. The pipeline of MTT begins with sampling a segment of expert trajectory, starting from $\theta^{E}_t$ to $\theta^{E}_{t+M}$ with $0\leq t \leq t+M \leq T$. Then, to generate a synthetic segment, a model, $\theta^S_t$, is initialised by, $\theta^{E}_t$, and trained on the learnable dataset, $\phi$, to get $\theta^{S}_{t+ N}(\phi)$ after $N$ iteration. Following the bilevel optimization terms, the disparity between $\theta^{S}_{t+ N}(\phi)$ and $\theta^{E}_{t+M}$ is optimized to learn the synthetic dataset forming the outer-loop with the optimization problem defined as:
\begin{align}
    \min_{\phi}& \, \mathcal{L}(\theta^S(\phi)) := \frac{1}{\delta}||\theta^{S}_{t+ N}(\phi) - \theta_{t+ M}^{E}||^2_2  \label{eq:mtt}\\
            \text{s.t. } & \theta^{S}_{t+ N}(\phi)  = \Xi_{N}(\theta^{S}_t, \phi) \nonumber
            \cut{
            & \theta^{S}_t  = \theta^{T}_t \\
            & \delta  =||\theta^E_t - \theta^E_{t+M}||_2^2
            }
\end{align} 
where $\Xi_N(\cdot)$ represents $N$ differentiable minimizing steps on the inner-loop objective, CrossEntropy loss, $\mathcal{L}_{CE}(\theta, \phi)$. The existing optimizers can instantiate the inner-loop, such as SGD whose one-step optimization is exemplified by $\Xi(\theta, \phi) = \theta - \alpha \nabla\mathcal{L}_{CE}(\theta, \phi)$ where $\alpha$ denotes the learning rate. Note $M$ and $N$ are not necessarily equal since dense information in the synthetic datasets leads to fast training. $\delta$ in Eq.~\ref{eq:mtt}, stabilising the numerical computation, can be unpacked as $||\theta^E_t - \theta^E_{t+M}||_2^2$.

\subsection{Sharpness-Aware Minimization}
\cut{
We introduce our method in this section starting with reviewing a DC framework, Matching Training Trajectory (MTT)~\citep{mtt}, applied in this work. Then we combine the bilevel optimization with sharpness-aware optimization tailored for dataset condensation with a loss landscape smoothing strategy for accurate sharpness approximation. To efficiently reduce the computational burden introduced by the sharpness-aware minimizers, we design and analyse time and memory-saving hypergradient approximations for the long horizon inner-loop with the general method outlined in Algorithm~\ref{general_algorithm}.
}
Given the training data, $D$, consider a training problem where the objective function is denoted as $\mathcal{L}(\phi; D)$ with the learnable parameter $\phi$, the objective function of SAM is framed as:
\begin{align}
   \min_{\phi} \max_{||\epsilon||_2 \leq \rho} \mathcal{L}(\phi + \epsilon; D),
\end{align}
where approximating sharpness is achieved by finding the perturbation vectors $\epsilon$ maximizing the objective function in the Euclidean ball with radius, $\rho$, with the sharpness defined as:  
\begin{align}
   \max_{||\epsilon||_2 \leq \rho} \big| \mathcal{L}(\phi + \epsilon; D) - \mathcal{L}(\phi ; D) |.
\end{align}
Instead of solving this problem iteratively, a closed-form approximation of the optimality by utilisation of the first-order Taylor expansion of the training loss is given by 
\begin{align*}
   \epsilon = \rho \frac{\nabla \mathcal{L}(\phi)}{|| \nabla \mathcal{L}(\phi)||_p} \approx \argmax_{||\epsilon||\, \leq \rho}\mathcal{L}(\phi + \epsilon).
\end{align*}
Overall, the updating procedure of SAM in each iteration is summarised as follows: 
\begin{align}
 \phi  = \phi - \alpha \nabla\mathcal{L}(\phi +\epsilon) \quad \textbf{s.t.} \,\,\, \epsilon = \rho \frac{\nabla \mathcal{L}(\phi)}{|| \nabla \mathcal{L}(\phi)||_p},
\end{align} 
where $\alpha$ represents the learning rate and after computing the gradient, $\nabla\mathcal{L}(\phi +\epsilon)$, the parameter update procedure is instantiated by standard optimizers, such as SGD and Adam~\citep{kingma2014adam}. Without losing generality, we set $p = 2$ for simplicity for the rest of this work. Due to the two-stage gradient calculation at $\phi$ and $\phi + \epsilon$, the computational overhead of SAM is doubled.

\begin{algorithm}[t]
\small
\caption{\ourModel{} for dataset condensation.}
\label{general_algorithm}
\begin{algorithmic}[1]
\STATE {\bfseries Input: } $\{\theta^E_t\}^T_0$, $\alpha$, $\beta$.
\STATE {\bfseries Output: } $\phi$ 
\STATE Init $\phi$
\WHILE{not converged or reached max steps}
    \STATE Sample an iteration $t$ to construct an expert segment, $\theta^E_t$, and $\theta^E_{t+M}$
    \STATE $\theta^S = \theta^E_{t}$  
    \STATE $\phi^{\Delta}_j \sim \mathcal{N}(0, \gamma ||\phi_j||_2 I )$
    \STATE $\phi = \phi + \phi^{\Delta}$ 
    \FORALL{$i \leftarrow 1 $ to $ N $ }
    \STATE $\theta^S = \theta^S - \alpha \nabla \mathcal{L}(\theta^S, \phi)$
    \ENDFOR
    \STATE Compute $\nabla F(\phi)$ by Eq.~\ref{eq:first_hyper}
    \STATE $\epsilon = \rho \nabla F(\phi)/||\nabla F(\phi)||_2$
    \STATE $\bar{\theta}^S = \theta^S_{t+\kappa}$ 
    \STATE $\phi = \phi - \phi_{\Delta}$
    \FORALL{$i \leftarrow N- \tau $ to $ N $ } 
    \STATE $\bar{\theta}^S = \bar{\theta}^S - \alpha \nabla \mathcal{L}(\bar{\theta}^S, \phi+\epsilon)$ 
    \ENDFOR
    \STATE Compute $\nabla F(\phi+\epsilon)$ by Eq.~\ref{eq:second_hyper}
    \STATE $\phi  = \phi - \beta \nabla F(\phi+\epsilon)$
\ENDWHILE
\end{algorithmic}
\end{algorithm}

\section{Method}
We introduce our method in this section by starting with configuring the trajectory matching-based dataset condensation under the Sharpness-Aware Bilevel optimization framework while handling the inaccurate sharpness approximation. Then, two strategies with mathematical support are proposed to reduce the computation cost introduced by the vanilla application of SAM. Additionally, we further boost the computational efficiency with a closed-form solution for learning rate learning instead of the backpropagation through inner-loop unrolling. The general idea is summarised in Algorithm~\ref{general_algorithm}. 
\subsection{Smooth Sharpness-Aware Minimization for Dataset Condensation}
generalizing to the unseen tasks is challenging for the learned synthetic datasets. To mitigate this issue, we steer the optimization on the outer-loop in Eq.~\ref{eq:mtt} and minimize the objective function forward landing in the flat loss landscape region to enable the synthetic data to be generalized to both in- and out-of-domain settings. This property has been studied in~\citep{petzka2021relative, kaddour2022flat}, in the uni-level optimization. In this work, we forage this into the bilevel optimization framework by integrating Shaprness-Aware minimization. 
To jointly optimize the sharpness of the outer-loop and the distance between the trajectory w.r.t to the synthetic dataset, we maximize the objective function in the $\rho$ regime for the sharpness proxy approximation and then optimize the distance between trajectories according to the gradient computed on the local maximum for the dataset learning. This process is described as follows: 
\begin{align}
   \min_{\phi}\max_{||\epsilon||_2 \leq \rho} & \mathcal{L}( \theta^S (\phi +\epsilon)) =  \frac{1}{\delta} ||\theta^{S}_{t+ N}(\phi+\epsilon) - \theta_{t+ M}^{E}||^2_2  \label{eq:flat_loss} \\
   \text{s.t. } & \theta^{S}_{t+ N}(\phi)  = \Xi_{N}(\theta^{S}_t, \phi) .
\end{align}
We define $F(\phi) = \mathcal{L}( \theta^S_{t+N} (\phi))$ to eliminate the effect of the inner-loop solution on the outer-loop loss value without losing generality. The perturbation vector, $\epsilon$, is computed through a closed-form solution derived through the first-order Taylor expansion of the objective function in Eq.~\ref{eq:mtt}.
\begin{align}
\epsilon = & \argmax_{||\epsilon||_2 \leq \rho} \mathcal{L}( \theta^S (\phi +\epsilon)) = \argmax_{||\epsilon||_2 \leq \rho} F(\phi+\epsilon) \nonumber \\
& \approx \argmax_{||\epsilon||_2 \leq \rho} F(\phi) + \epsilon \cdot \nabla F(\phi) \nonumber \\
&  = \argmax_{||\epsilon||_2 \leq \rho} \epsilon \cdot \nabla F(\phi) \approx \rho \frac{\nabla F(\phi)}{||\nabla F(\phi)||_2}  \label{pertub}.
\end{align}
The closed-form solution given in Eq.~\ref{pertub} can be interpreted as a one-step gradient ascent. However, this one-step gradient ascent may fail to reach the local maximum of the sharpness proxy, due to the high variance of hypergradient caused by the complicated outer-loop loss landscape. This phenomenon has also been observed by~\citep{liu2022random, esam} in the uni-level optimization and will aggravate in the complicated bilevel case~\citep{sharpmaml}. To conduct accurate sharpness approximation, motivated by~\citep{liu2022random, haruki2019gradient, wen2018smoothout, duchi2012randomized}, we introduce fluctuation on the learnable dataset to smooth the landscape. To be more specific, each synthetic image indexed by $j$ is perturbed by a random noise sampled from a Gaussian distribution with a diagonal covariance matrix whose magnitude is proportional to the norm of each image $||\phi_j||$:  
\begin{align*}
\phi_j = \phi_j + \phi_j^{\Delta},\quad \phi_j^{\Delta} \sim \mathcal{N}(0, \gamma ||\phi_j||_2),
\end{align*}
where $\gamma$ is a tunable hyperparameter controlling the fluctuation strength. This process is conducted on the image independently in each one-step gradient ascent.

\subsection{Efficient Sharpness-Aware Minimization in Bilevel Optimization}
One can notice that a one-step update in the outer-loop needs to compute the hypergradient twice with one for the perturbation vector $\epsilon$ and the other for the real update gradient, $\nabla F(\phi)$. Directly computing those two gradients will double the computation cost in contrast with MTT and FTD instead of TESLA which we will discuss later. To alleviate this problem, we proposed two approximation strategies, Truncated Unrolling Hypergradient (TUH) and Trajectory Reusing (TR). 

\textbf{Truncated Unrolling Hypergradient.} The long inner-loop horizon introduces tremendous computational overhead. In our dataset condensation framework, the hypergradient for updating the learnable dataset is computed by differentiating through the unrolled computational graph of the inner-loop. This vanilla hypergradient computation lets the memory cost scale with the number of the inner-loop iterations which is not feasible as condensing the complicated datasets requires long horizon inner-loops. Instead, we \emph{truncate the backpropagation} by only differentiating through the last several steps of the inner-loop. This reduces both the required memory and computational time. More concretely, the truncated hypergradient computation with $N$ step unrolling can be expressed as:
\cut{
\begin{align}
    \frac{\partial F_{\iota}(\phi)}{\partial \phi} = & \frac{\partial \mathcal{L}(\theta(\phi))}{\partial \theta_{\iota}}\frac{\partial \theta_{\iota}}{\partial \phi} \\ 
    = &\sum^{N}_{i = \iota} \frac{\partial \mathcal{L}(\theta(\phi))}{\partial \theta_{N}} \Bigg( \prod^{N}_{i'= i} \frac{\partial \theta_{i'}}{\partial \theta_{i'-1}} \Bigg) \frac{\partial \theta_i}{\partial \phi} , \label{eq:first_hyper}
\end{align}
}
\begin{align}
    \frac{\partial F_{\iota}(\phi)}{\partial \phi} = \sum^{N}_{i = \iota} \frac{\partial \mathcal{L}(\theta(\phi))}{\partial \theta_{N}} \Bigg( \prod^{N}_{i'= i} \frac{\partial \theta_{i'}}{\partial \theta_{i'-1}} \Bigg) \frac{\partial \theta_i}{\partial \phi} , \label{eq:first_hyper}
\end{align}
where $\iota$ controls the number of truncated steps that $N-\iota$ steps of the inner-loop will be differentiated through. In addition, the risk of hypergradient exploding and vanishing caused by the ill-Jacobian $\frac{\partial \theta_i}{\partial \theta_{i-1}}$, which may happen in any inner-loop step, can be reduced. This mechanism can be easily implemented by releasing the computational graph while optimising the inner-loop and then creating the computational graph at a certain iteration with PyTorch-based pseudocode given in Appx.~\ref{pytorch_psed}.

\cut{Following~\citep{truncated_gradient,onestep_gradient},} 
We analyse the discrepancy between hypergradients computed by the truncated and untruncated computational graph in the setting where the synthetic trajectory is produced by optimized from the initialisation $\theta^E_0$ until converge. 
\begin{restatable}{proposition}{gradientSimiliarity}  \label{prop:gradient_similiarity}
Assmue $\mathcal{L}_{CE}$ is $K$-smooth, twice differentiable, and locally $J$-strongly convex in $\theta$ around $\{\theta_{\iota+1},..., \theta_N \}$. Let $\Xi(\theta, \phi) = \theta - \alpha \nabla\mathcal{L}_{CE}(\theta, \phi)$. For $\alpha \leq \frac{1}{K}$, then
\begin{align*}
&\left\| \frac{\partial F(\phi)}{\partial \phi} - \frac{\partial F_{\iota}(\phi)}{\partial \phi} \right\|  \\ 
&\leq 2^{\iota} (1-\alpha J)^{N-\iota+1} \left\| \frac{\partial \mathcal{L}(\theta(\phi))}{\partial \theta_{N}(\phi)} \right\| \max_{i \in \{0 , .. \iota\}}\left  \| \frac{\partial \theta_i}{\partial \phi}\right\|,
\end{align*}
where $\frac{\partial F(\phi)}{\partial \phi}$ denotes the untruncated hypergradient.
\end{restatable}
The Proposition~\ref{prop:gradient_similiarity} shows that the error of the truncated hypergradient decreases exponentially in $N- \iota + 1$ when $\theta$ converges to the neighbourhood of a local minimum in the inner-loop and the proof is given in Appx.~\ref{proof:gradient_similiarity}. 

\textbf{Trajectory Reusing.} 
The sharpness-aware minimization requires computing the gradient twice for sharpness proxy approximation and free parameter update, which means in bilevel optimization the inner-loop is required to unroll twice. This boosts the computational spending and slows down the training speed when inner-loops comprise long trajectories. To improve the efficiency of training by benefiting from the existing knowledge, we propose to reuse the trajectory generated by the first round of inner-loop unrolling. We denote the trajectories generated by training on the perturbed dataset as $\hat\theta_i(\phi+\epsilon)$. Other than unrolling the entire second trajectory initialised by the expert segment, the training is initialised by the middle point, indexed by $\tau$, from the first trajectory $\hat{\theta}_{\tau}(\phi + \epsilon):= \theta_{\tau}(\phi)$. Note that the hypergradient for the dataset update is truncated implicitly since this hypergradient approximation will not consider the steps earlier than $\tau$ which is further constrained, $\tau \geq \iota$. Coupled with the same truncated strategy for the first round, the hypergradient in the second trajectory is computed as:
\cut{
\begin{equation}
    \frac{\partial F_{\tau,\epsilon}(\phi)}{\partial \phi}=  \frac{\partial \mathcal{L}(\theta(\phi))}{\partial \theta_{\tau}}\frac{\partial \theta_{\tau}}{\partial \phi} = \sum^{N}_{i = \tau} \frac{\partial \mathcal{L}(\theta(\phi))}{\partial \theta_{N}} \Bigg( \prod^{N}_{i'= i} \frac{\partial \theta_{i'}}{\partial \theta_{i'-1}} \Bigg) \frac{\partial \theta_i}{\partial \phi}\Bigg|_{\phi = \phi + \epsilon, \,\, \hat{\theta}_{\tau}(\phi+\epsilon) = \theta_{\tau}(\phi)} \label{eq:second_hyper}
\end{equation}
}
\begin{align}
    &\frac{\partial F_{\tau,\epsilon}(\phi)}{\partial \phi}  \nonumber \\ 
    & = \sum^{N}_{i = \tau} \frac{\partial \mathcal{L}(\theta(\phi))}{\partial \theta_{N}} \Bigg( \prod^{N}_{i'= i} \frac{\partial \theta_{i'}}{\partial \theta_{i'-1}} \Bigg) \frac{\partial \theta_i}{\partial \phi}\Bigg|_{\phi = \phi + \epsilon, \,\, \hat{\theta}_{\tau}(\phi+\epsilon) = \theta_{\tau}(\phi)}. \label{eq:second_hyper}
\end{align}
One may notice that the trajectory reusing strategy assumes the difference between two trajectories before step $\tau$ can be ignored. To rigorously study the effect of this assumption, we analyse the distance between $\theta_{\tau}(\phi)$ and $\theta_{\tau}(\phi+\epsilon)$. Similar to the Growth recursion lemma~\citep{hardt2016train} applied to upper-bound the difference between two weight points of two different trajectories trained by the dataset with only one data point difference. We develop the bound for the difference between two weight points at the same iteration of their trajectories generated by the datasets with and without perturbation below. The proof is provided in Appx.\ref{proof:traj_difference}. 
\begin{restatable}{theorem}{trajDifference} 
\label{theorem:traj_difference}
Let $\mathcal{L}(\phi, \theta)$ be a function that is $\sigma$-smooth and continuous with respect to its arguments $\phi$ and $\theta$. Additionally, let the second-order derivatives $\nabla_{\phi} \nabla_{\theta} \mathcal{L}(\phi, \theta)$ be $\beta$-continuous. Consider two trajectories obtained by conducting gradient descent training on the datasets $\phi$ and $\phi + \epsilon$, respectively, with a carefully chosen learning rate $\alpha$ and identical initializations. After $\tau$ steps of training, let $\Delta \theta_\tau = \hat{\theta}_\tau(\phi+\epsilon) - \theta_\tau(\phi)$. Then, we have:
\begin{align*}
\left\|\Delta \theta_{\tau} \right\| \leq \alpha \tau (2\sigma + \beta \rho).
\end{align*}
\end{restatable}
This theorem tells us that the bound of the distance of those two points is associated with the learning rate and the number of iterations. Thus, when the learning rate and $\tau$ are selected reasonably, $\theta_{\tau}(\phi)$ approximate $\hat{\theta}_{\tau}(\phi + \epsilon)$ properly. In addition, we set $\tau = \iota$ in our experiments to reduce the hyperparameter tuning efforts even though tuning them separately may achieve better results. We compare the time and memory complexity of our method and Reverse Model Reverse Mode Differentiation (RMD) used in MTT~\citep{mtt} and FTD~\citep{ftd} in Table~\ref{tab:computational_complexity} to exhibit the efficiency provided by our method.
\cut{
\begin{table}[t]
    \centering
    \caption{The computational complexity comparison for the sharpness-aware based bilevel optimization. The approaches include fully unrolling hypergradient which is further divided into Forward Mode Differentiation (FMD) and Reverse Mode Differentiation (RMD) \citep{franceschi2017forward}, and TUH in time and memory cost. $c$ is the time cost for computing $\Xi(\theta, \phi )$ with $\theta \in R^{P}$ and $\phi \in R^{Q}$. P and Q denote the dimensions of the base model and synthetic dataset}
    \label{tab:computational_complexity}
    \begin{tabular}{c | c c }
    \hline
      Methods   &   Time  & Memory  \\
    \hline
      FMD       &   $\mathcal{O}(2cPN )$  & $\mathcal{O}(2PQ)$     \\
      RMD       &   $\mathcal{O}(2cN )$   & $\mathcal{O}(2PN)$     \\
      TUH       &   $\mathcal{O}(cN + c \tau )$  &  $\mathcal{O}(P(N-\iota) + P(N-\tau))$   \\
    \hline
    \end{tabular}
\end{table}
}
\begin{table}[t]
    \centering
    \caption{The computational complexity comparison for different trajectory matching based algorithms in time and memory cost. $c$ is the time cost for computing $\Xi(\theta, \phi )$ with $\theta \in R^{P}$ and $\phi \in R^{Q}$. P and Q denote the dimensions of the base model and synthetic dataset.}
    \label{tab:computational_complexity}
    \begin{tabular}{c | c c }
    \hline
      Methods   &   Time  & Memory  \\
    \hline
      MTT, FTD                 &   $\mathcal{O}(cN )$           & $\mathcal{O}(PN)$     \\
      TESLA                    &   $\mathcal{O}(2cN )$          & $\mathcal{O}(P)$     \\
      TUH + TR                 &   $\mathcal{O}(cN + c \tau )$  &  $\mathcal{O}(P(N-\iota))$   \\
    \hline
    \end{tabular}
\end{table}

\begin{align}
\alpha = \alpha - \lambda \frac{\partial \mathcal{L}(\theta_N(\phi))}{\partial \theta_{N}} \cdot \Bigg( - \sum^{N-1}_{i = 0} \frac{\partial \mathcal{L}_{CE}(\theta_i, \phi)}{\partial \, \theta_{i}}\Bigg), \label{eq:lr_learning}
\end{align}
where $\lambda$ indicates the learning rate for the learning rate learning. This closed-form solution only aggregates the gradient of each step which only requires first-order derivative computation. We compare the learning rate learning dynamic produced by first-order (our method) and second-order derivatives, demonstrating limited differences between those two methods. The derivation and the details of experiments are given in Appx.~\ref{lr_fod}.

In essence, \ouracronym{} is designed to conduct efficient sharpness minimization in the outer-loop of the bilevel optimization-based dataset condensation methods and the proposed efficiency strategies, including THU and TR, are flexible enough to adapt to other advanced sharpness-aware optimizers such as ASAM~\citep{asam} and Vasson~\citep{vasso}. 

\begin{table*}[t]
    \centering
    \caption{Test Accuracy (\%) Comparison of different image per category (IPC) setting on Cifar10, Cifar-100 and Tiny ImageNet: the models are trained on the syntactic dataset learned by MTT and our method independently and evaluated on the corresponding test set with real images. We cite the results of DC, DM and MMT from FTD~\citep{ftd}.} 
     \label{tab:indomain}
    \resizebox{0.85\textwidth}{!}{
    \begin{tabular}{c c | c c c c c c c c}
    \toprule
      Method      & IPC    & DC & DSA & DM   & MTT       & FTD    & TESLA    & MDC      & Ours  \\
\midrule
      Cifar-10    &  1     & $28.3_{\pm0.5}$ & $28.8_{\pm0.7}$ & $26.0_{\pm0.8}$ & $46.2_{\pm0.8}$ & $46.8_{\pm0.3}$           & $48.5_{\pm0.8}$ &  $47.5_{\pm0.4}$ & $\textbf{49.0}_{\pm0.3}$ \\
                  &  10    & $44.9_{\pm0.5}$ & $52.1_{\pm0.6}$ & $48.9_{\pm0.6}$ & $65.4_{\pm0.7}$ & $66.6_{\pm0.3}$           & $66.4_{\pm0.8}$ &  $66.7_{\pm0.7}$ & $\textbf{67.1}_{\pm0.3}$ \\
                  &  50    & $53.9_{\pm0.5}$ & $60.6_{\pm0.5}$ & $63.0_{\pm0.4}$ & $71.6_{\pm0.2}$ & $73.8_{\pm0.3}$           & $72.6_{\pm0.7}$ &  $73.7_{\pm0.3}$ & $\textbf{73.9}_{\pm0.2}$ \\
\midrule
      Cifar-100   &  1     & $12.8_{\pm0.3}$ & $13.9_{\pm0.3}$ & $11.4_{\pm0.3}$ & $24.3_{\pm0.3}$ & $25.2_{\pm0.2}$           & $24.8_{\pm0.4}$ & $25.9_{\pm0.2}$  & $\textbf{26.1}_{\pm0.4}$ \\
                  &  10    & $25.2_{\pm0.3}$ & $32.3_{\pm0.3}$ & $29.7_{\pm0.3}$ & $39.7_{\pm0.4}$ & $\textbf{43.4}_{\pm0.3}$  & $41.7_{\pm0.3}$ & $42.7_{\pm0.6}$  & $43.1_{\pm0.5}$ \\
                  &  50    & -               & $42.8_{\pm0.4}$ & $43.6_{\pm0.4}$ & $47.7_{\pm0.2}$ & $50.7_{\pm0.3}$  & $47.9_{\pm0.3}$ & $49.6_{\pm0.4}$  & $\textbf{50.9}_{\pm0.5}$ \\
\midrule
      TinyImageNet&  1     & -               & -               & $3.9_{\pm0.2}$  &  $8.8_{\pm0.3}$ & $10.4_{\pm0.3}$           & $7.8_{\pm0.2}$ & $9.9_{\pm0.2}$    & $\textbf{10.9}_{\pm0.2}$   \\
                  &  10    & -               & -               & $12.9_{\pm0.4}$ & $23.2_{\pm0.2}$ & $24.5_{\pm0.2}$           & $20.8_{\pm0.9}$ & $24.8_{\pm0.4}$  & $\textbf{25.4}_{\pm0.4}$     \\
                  &  50    & -               & -               & $24.1_{\pm0.3}$ & $28.0_{\pm0.3}$ & $28.2_{\pm0.3}$           & $27.8_{\pm1.1}$ & $28.1_{\pm0.2}$  & $\textbf{29.4}_{\pm0.3}$     \\
\cut{
\midrule
ImageNet-1K    & 1   &  - & - & - & - & -        & $7.7_{\pm0.2}$    &-        & $\textbf{8.2}_{\pm0.4}$     \\ 
                      & 10   &  - & - & - & - & -         & $17.8_{\pm1.3}$    &-        & $\textbf{18.5}_{\pm0.9} $    \\ 
                      & 50   &  - & - & - & - & -         & $27.9_{\pm1.2}$    &-        & $\textbf{28.4}_{\pm1.1}$     \\ 
}
    \bottomrule
    \end{tabular}
   } 
\end{table*}
\section{Experiments}
We evaluate \ouracronym{} on various in-domain tasks where the neural architecture and data distribution on the dataset learning and test stage are the same with different datasets and different numbers of images per category (IPC). Besides, cross-architecture and cross-task evaluation are conducted to demonstrate the generalization achieved in sharpness minimization in out-of-domain settings. All experiment configurations including the dataset and architecture details are given in Appx.~\ref{experiment_setting_details}.

\subsection{Main Results}
\textbf{Popular Benchmark.} We compare our method against the other dataset condensation techniques, such as DC~\citep{dc2021}, DSA~\citep{dsa2021}, DM~\citep{dm}, MTT\citep{mtt}, FTD~\citep{ftd}, TESLA~\citep{TESLA} and MDC~\citep{he2024multisize}. The results from Table~\ref{tab:indomain} demonstrate the benefits of the flat minima that \ouracronym{} outperforms the competitors on almost all the settings of the standard dataset condensation benchmarks with various IPCs. This benefit can be further observed in the high-resolution image condensation task in Table~\ref{tab:ImageNet}. Note that in our case, we merely build \ouracronym{} up on Vanilla MMT~\citep{mtt} without integrating the flat trajectory trick in FTD and the soft label in TESLA. Still, there are clear improvement margins over other trajectory-matching-based DC competitors. 
\cut{
Limited by the computational resource, we cannot conduct full batch training on Cifar100 with 10 IPC, 50 IPC and Tiny ImageNet with 10 IPC as that utilised on MTT and FTD, which we believe is the main reason that \ouracronym{} performs slightly worse than FTD on the Cifar100 with 10 IPC setting. Besides, there are clear improvement margins over other trajectory-matching-based DC competitors. }
\begin{table}[t]
    \centering
    \small
    \caption{Test accuracy (\%) comparison on the Subsets of ImageNet including ImageNette, ImageWoor, ImageFruit and ImageNeow with high-resolution ($128\times128$): All the syntactic datasets are learned and tested on ConvNet with 10 IPC.}
    \resizebox{\linewidth}{!}{\begin{tabular}{ccccc}
\toprule
                  & ImageNette & ImageWoof & ImageFruit & ImageMeow               \\ \midrule
MTT           & $63.0_{\pm 1.3}$   & $35.8_{\pm1.8}$   & $40.3_{\pm1.3}$   & $40.4_{\pm2.2}$   \\
FTD           & $67.7_{\pm 0.7}$   & $38.8_{\pm1.4}$   & $44.9_{\pm1.5}$   & $43.3_{\pm0.6}$  \\ 
Ours          & $\mathbf{68.2}_{\pm 0.5}$   & $\mathbf{39.4}_{\pm1.2}$   & $\mathbf{45.2}_{\pm1.3}$   & $\mathbf{45.4}_{\pm0.9}$  \\ 
\midrule
All           & $87.4_{\pm 1.0}$   & $67.0_{\pm1.3}$   & $63.9_{\pm2.0}$   & $66.7_{\pm1.1}$  \\ 
\bottomrule
    \label{tab:ImageNet}
    \end{tabular}
    }
\end{table}

\textbf{ImageNet Comparison with TESLA.} Due to the high memory cost of the trajectory matching-based dataset condensation methods, most existing works fail to distil synthetic datasets from ImageNet~\citep{imagenet}. TESLA~\citep{TESLA} trades off time complexity and performance in its two inner loops to maintain a constant memory cost, which is the first trajectory-matching work enabling the ImageNet condensation. Our method has a similar training protocol with TESLA, as both require executing the inner-loop twice to execute outer-loop updates. However,  In contrast, our model also achieves constant memory usage by differentiating through the last N steps of the inner-loop thanks to the provable hypergradient approximation error bound. Moreover, it requires only a partial second inner-loop execution and aims to converge into a flat loss region improving the generalization of synthetic data significantly, outperforming TESLA even without relying on soft-label fitting tricks in terms of both generalization ability and the running time cost, shown in Table~\ref{tab:TESLA_satm_time_accuracycomparison}. 
\begin{table}[t]
\centering
\small
\caption{Accuracy Comparison of TESLA and SATM across different IPCs on ImageNet-1K and average time cost (sce) comparison with 50 inner-loop iterations and 50 IPCs.}
\begin{tabular}{c|ccc|c}
\toprule
Model/IPC &  1                 &  10           &  50     &   Time Cost                      \\ \midrule
TESLA     &  $7.7_{\pm0.2}$       &  $17.8_{\pm1.3}$        &  $27.9_{\pm1.2}$ & $46.8_{\pm0.3}$   \\ 
SATM      &  $\textbf{8.9}_{\pm0.3}$  & $\textbf{19.2}_{\pm0.9}$   & $\textbf{29.2}_{\pm1.1}$ & $\textbf{44.7}_{\pm0.2}$ \\
\bottomrule
\end{tabular}
\label{tab:TESLA_satm_time_accuracycomparison}
\end{table}

\textbf{Cross Architecture.} In this work, we are also interested in studying whether the advantages brought by the flatness can be observed in cross-architecture tasks, which leads to numerous practical applications. In Table~\ref{tab:across_architecture}, the synthetic datasets by learned \ouracronym{} for Cifar10 exhibit strong generalization ability across the unseen architectures on both IPC 10 and 50 settings over the candidate architectures in comparison with those learned by MTT~\citep{mtt}, FTD~\citep{ftd}. Additionally, one can notice that the performance of the learned dataset from the in-domain setting is not guaranteed in the cross-architecture setting. For instance, FTD performs similarly to \ouracronym{} in the Cifar10 with 10 and 50 IPC settings when deploying on ConvNet in the dataset learning stage. However, the performance gaps become remarkable once the same datasets are used across architectures.
\begin{table}[t]
    \centering
    \small
    \caption{Test accuracy (\%) comparison on Cifar10 with 10 and 50 images per class setting: the syntactic datasets by MTT, FTD and our algorithm are learned on ConvNet and tested on AlexNet, VGG11 and ResNet18.}
    \resizebox{\linewidth}{!}{
    \begin{tabular}{c c| c c c c}
    \toprule
      Methods   &   IPC  & ConvNet & AlexNet & VGG11 &  ResNet18  \\
    \hline
      MTT    & \multirow{ 3}{*}{10}   &  $64.3_{\pm0.7}$            & $34.2_{\pm2.6}$           & $50.3_{\pm0.8}$           & $46.4_{\pm0.6}$    \\
      FTD                    &        &  $66.6_{\pm0.4}$            & $36.5_{\pm1.1}$           & $50.8_{\pm0.3}$           & $46.2_{\pm0.7}$    \\
      Ours                   &        &  $\textbf{67.1}_{\pm0.5}$   & $\textbf{37.8}_{\pm0.8}$  & $\textbf{51.4}_{\pm0.3}$  & $\textbf{47.7}_{\pm0.4}$  \\
    \hline
      MTT          & \multirow{ 3}{*}{50}       &  $71.6_{\pm0.2}$          & $48.2_{\pm1.0}$           & $55.4_{\pm0.8}$  & $61.9_{\pm0.7}$    \\
      FTD                   &                   &  $73.8_{\pm0.2}$          & $53.8_{\pm0.9}$           & $58.4_{\pm1.6}$  & $65.7_{\pm0.3}$    \\
      Ours                  &                   &  $\textbf{74.2}_{\pm0.3}$ & $\textbf{56.9}_{\pm0.7}$  & $\textbf{63.5}_{\pm1.1}$  & $\textbf{66.1}_{\pm0.5}$    \\
    \bottomrule
    \end{tabular}
    }
    \label{tab:across_architecture}
\end{table}

\textbf{Continual Learning.} We expose the learned dataset to the task incremental setting, following the same protocol discussed in Gdumb~\citep{gdumb} for a fair comparison with datasets produced by competitors such as DM~\citep{dm}, MTT~\citep{mtt}, and FTD~\citep{ftd}. Typically, models encounter a sequence of data from different categories and lose access to data from previous categories after training. A limited memory budget is available to save dataset information from previous tasks, enabling models to retain gained knowledge while adapting to new tasks. In Figure~\ref{fig:continual_learning}, we show that at each stage, as new categories are received, our learned datasets consistently outperform others in three settings: 5-task incremental with 50 images per category on Cifar10, 10-and 20-task incremental with 3 IPC on Tiny ImageNet. Given the result in Fig~\ref{fig:continual_learning}, \ouracronym{} consistently outperforms other methods. \cut{whenever the models encounter new tasks on all the settings.}
\begin{figure*}[t]
    \centering
    \includegraphics[width=0.33\linewidth]{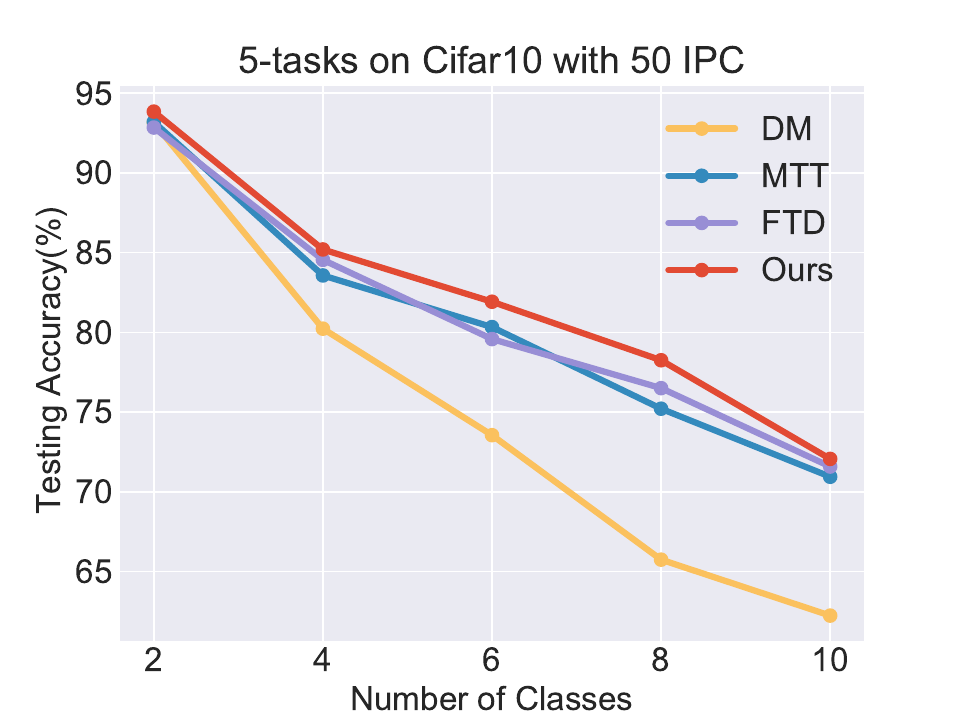}
    \includegraphics[width=0.33\linewidth]{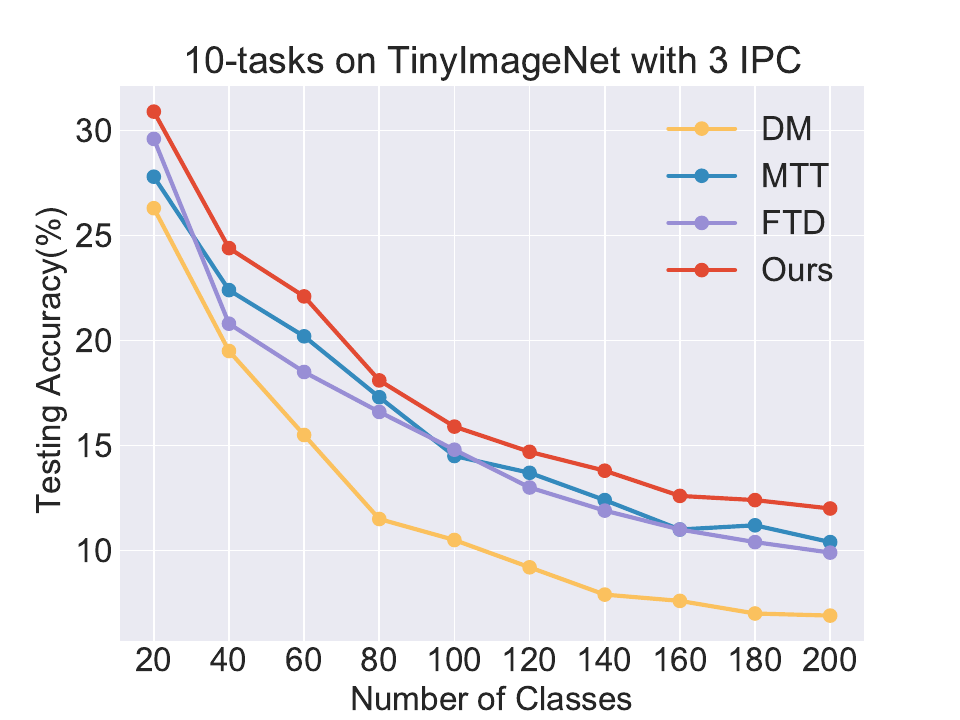}
    \includegraphics[width=0.33\linewidth]{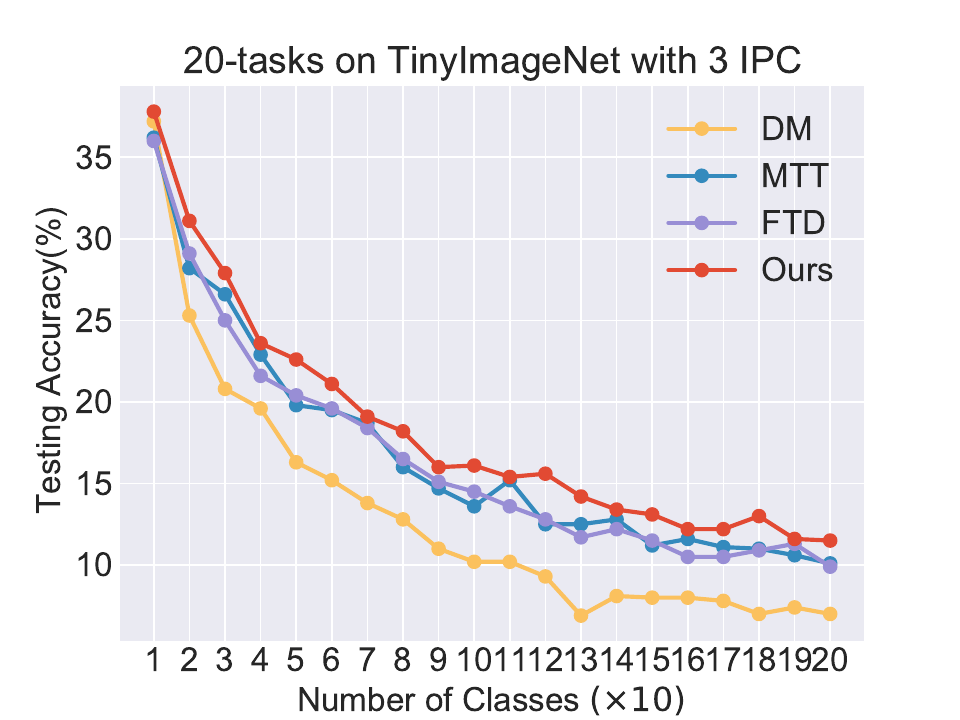}
\caption{Test accuracy (\%) comparison on continual learning. Left: 5-step class-incremental learning on Cifar10 50IPC, Middle: 10-step class-incremental learning on Tiny ImageNet 3IPC, Right: 20-step class-incremental learning on Tiny ImageNet 3IPC.}
\label{fig:continual_learning}
\end{figure*}

\subsection{Further Analysis}
\textbf{Computational Cost Comparison.}
We computed and recorded the memory and time costs when running SATM and then compared them with MTT and TESLA. The results were measured on a single NVIDIA A6000 GPU, except for MTT on ImageNet-1K~\citep{imagenet}, which required two A6000 GPUs.

In our experiments, at most only one-third of the inner-loop is retained to compute the hypergradients for sharpness approximation. Given the result in Table~\ref{tab:time_comparison}, our strategy significantly reduces memory consumption compared to MTT, enabling the dataset to be trained on a single A6000 GPU. Regarding time cost, SATM consistently outperforms the two inner-loop-based algorithms, TESLA. Additionally, SATM even consumes less time than MTT which requires retaining a full single inner-loop.
\begin{table}[t]
\centering
\small
\begin{tabular}{c|cc| cc}
\toprule
Model         & \multicolumn{2}{c}{CIFAR-100}          & \multicolumn{2}{c}{ImageNet-1K} \\ 
              &   Memory         &   Runtime       &      Memory  &     Runtime  \\ \midrule
MTT           & $17.1_{\pm0.1}$  & $12.1_{\pm0.6}$ & $26.6_{\pm0.1}$ & $45.9_{\pm0.5}$ \\
TESLA         & $ 3.6_{\pm0.1}$  & $15.3_{\pm0.5}$ & $26.6_{\pm0.1}$ & $47.4_{\pm0.7}$ \\
SATM          & $ 5.7_{\pm0.1}$  & $12.0_{\pm0.5}$ & $26.6_{\pm0.1}$ & $45.4_{\pm0.4}$ \\
\bottomrule
\end{tabular}
\caption{GPU memory (GB) and runtime (sec) comparison among MTT, TESLA and SATM on CIFAR100 and ImageNet-1K with results measured with a batch size of 100 and 50 inner-loop steps. }
\label{tab:time_comparison}
\end{table}

\textbf{Hypergradient Analysis.} To illustrate the effects of sharpness minimization on the process of synthetic dataset learning, we record the hypergradient norm of MTT and \ouracronym{} during training and report their mean and variance over training iterations. Depicted in Fig~\ref{fig:Hypergradient_Norm}, \ouracronym{} has a smaller mean and variance than MTT on Cifar100 with 3 IPC and Tiny ImnageNet 3IPC. Additionally, fewer spikes of hypergraident in \ouracronym{} can be observed, indicating more stable training. Moreover, the dynamic of the sharpness, measured by $\mathcal{L}(\phi+\epsilon) - \mathcal{L}(\phi)$, with decreasing trend shows that the synthetic dataset is landing into the flat loss region, which explains why our method enjoys better generalization ability.

\textbf{Aligning with Curriculum Learning.} DATM~\citep{datm} utilizes the difficulty of training trajectories to implement a curriculum learning-based dataset condensation protocol along with the soft label alignment. This approach is distinct from research focused on optimization efficiency and generalization, such as TESLA, FTD, and SATM which prioritize optimization efficiency through gradient approximation. In contrast, our work focuses on the difficulty of the loss landscape which is compatible with the curriculum learning strategy. We conduct experiments combining DATM's easy-to-hard training protocol and the soft label alignment with \ouracronym{} denoted as \ouracronym{}-DA, yielding the results in Table~\ref{tab:comparison_ipc_datm_satm}.
\begin{table}[t]
\centering
\small
\caption{Accuracy (\%) Comparison of DATM, and SATM-DA across various IPCs, datasets and configurations.}
\begin{tabular}{lc|cccc}
\toprule
                & IPC & DATM & SATM-DA \\ \midrule
\cut{
                & 1   & $46.9_{\pm0.5}$ & $\mathbf{48.6}_{\pm0.4}$ \\ 
CIFAR-10        & 10  & $66.8_{\pm0.2}$ & $\mathbf{68.1}_{\pm0.3}$ \\ 
                & 50  & $76.1_{\pm0.3}$ & $\mathbf{76.4}_{\pm0.6}$ \\ \midrule
}
                & 1   & $27.9_{\pm0.2}$ & $\mathbf{28.2}_{\pm0.8}$ \\ 
CIFAR-100       & 10  & $47.2_{\pm0.4}$ & $\mathbf{48.3}_{\pm0.4}$ \\ 
                & 50  & $55.0_{\pm0.2}$ & $\mathbf{55.7}_{\pm0.3}$ \\ \midrule
Tiny-ImageNet   & 1   & $\mathbf{17.1}_{\pm0.3}$ & $16.4_{\pm0.4}$ \\ 
                & 10  & $31.1_{\pm0.3}$ & $\mathbf{32.3}_{\pm0.6}$ \\ 
                & 50  & $39.7\pm0.3$ & $\mathbf{40.2}\pm0.7$ \\ 
\bottomrule
\end{tabular}
\label{tab:comparison_ipc_datm_satm}
\end{table}

\begin{figure}[!h]
    \centering
    \includegraphics[width=0.7\linewidth]{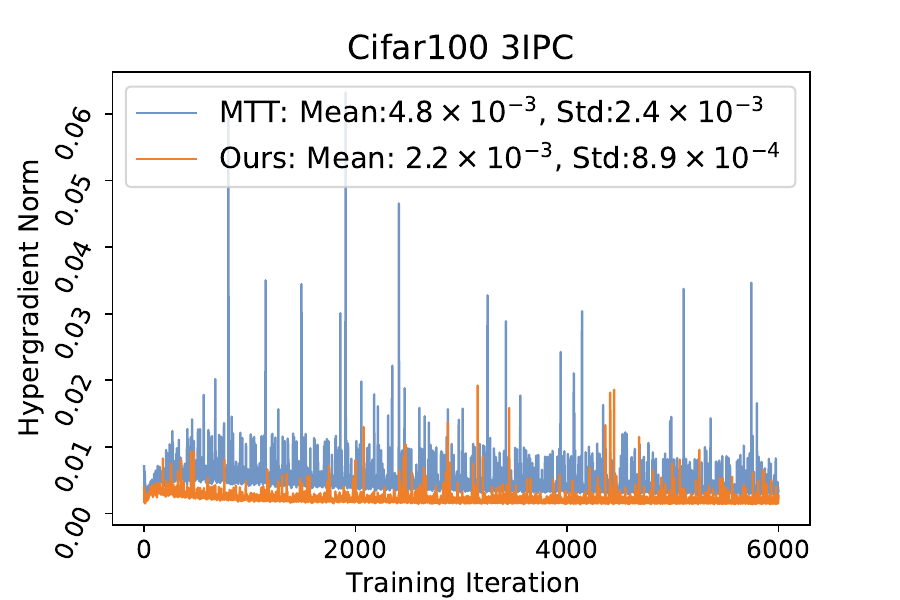}
    \includegraphics[width=0.7\linewidth]{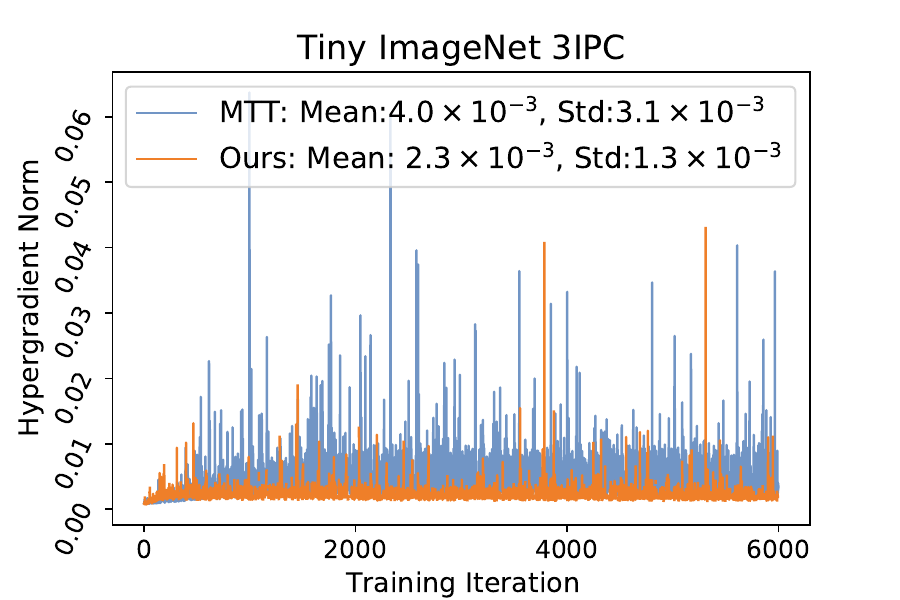}
    \includegraphics[width=0.7\linewidth]{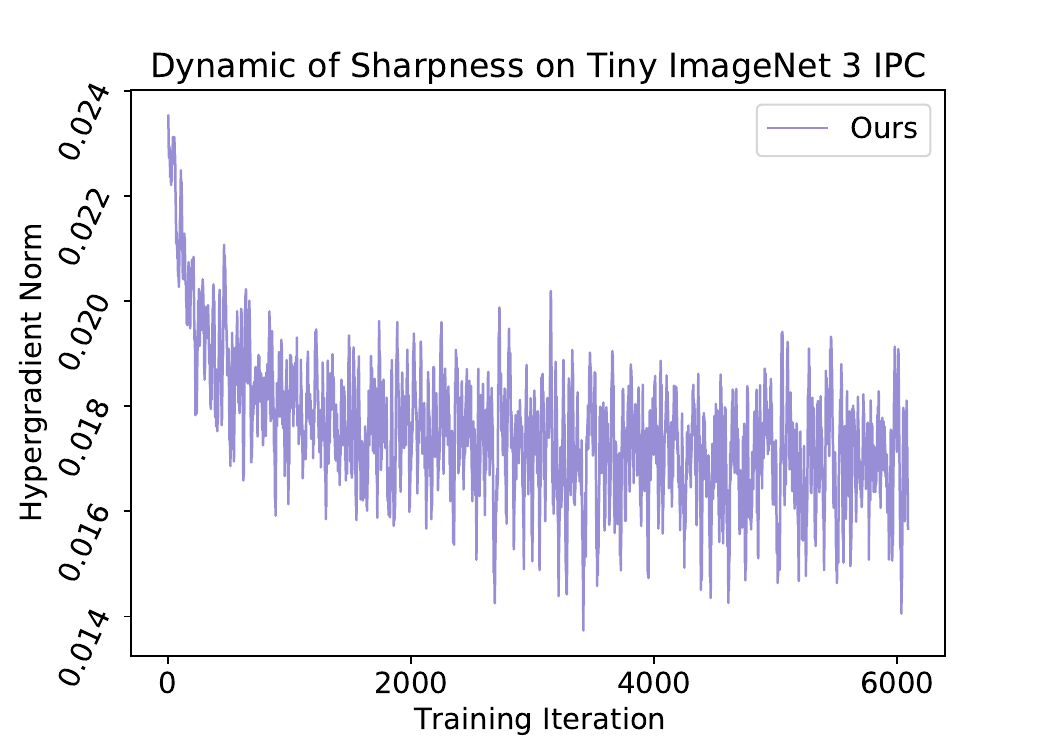}
\caption{Sharpness analysis by visualisation. Hypergradient Norm comparison between MTT and \ouracronym{}. Top: the hypergradient norm on Cifar100 with 10 IPC; Middle: the hypergradient norm on Tiny ImageNet with 3 IPC. Bottom: Sharpness dynamic on Tiny ImageNet with 3 IPC.}
\label{fig:Hypergradient_Norm}
\end{figure}

\textbf{Truncated Step Study.} We study the effects of the number of inner-loop steps remaining for the hypergradient estimation on the model performance with the settings that require the long inner-loops for dataset learning. Table~\ref{tab:setting_steps_comparison} details the settings, including the dataset, the number of images per category (IPC), and the inner-loop steps $N$. To analyze such effects, we retained the last $\frac{1}{k}$ steps, where $k =2,3,4,5,6$, of the total inner-loop steps. For simplicity, the inner-loop steps remained for the first round of hypergradient computation and trajectory reusing in the second round is kept the same which is applied across all experiments. The operation $int(\frac{N}{k})$ is used to determine the remaining inner-loop steps. We examined how accuracy changes with the remaining inner-loop steps by executing SATM for 10000 training iterations. A clear trend emerged: performance improves as the number of truncated iterations decreases and converges once the differentiation steps reach a certain threshold.
\begin{table}[t]
\centering
\small
\caption{Accuracy (\%) change along with the truncated inner-loop step change. We conduct SATM on CIFAR-10 to learn 1 image per category while running 50 iterations for the inner-loop and 80 iterations for the CIFAR-100 with 50IPC.}
\begin{tabular}{l|ccccc}
\toprule
\cut{
           & $\frac{1}{6}$ & $\frac{1}{5}$ & $\frac{1}{4}$ & $\frac{1}{3}$ & $\frac{1}{2}$ \\ \midrule
CIFAR-10   & $45.2_{\pm0.8}$ & $48.8_{\pm0.7} $& $47.5_{\pm0.3}$ & $49.0_{\pm0.3}$ & $49.2_{\pm0.2}$ \\ 
CIFAR-100  & $23.4_{\pm0.7}$ & $33.4_{\pm0.6} $& $48.7_{\pm0.6}$ & $50.9_{\pm0.5}$ & $50.5_{\pm0.3}$ \\ 
}
Configuration   & $\frac{1}{6}$ & $\frac{1}{5}$ & $\frac{1}{4}$ & $\frac{1}{3}$ & $\frac{1}{2}$ \\ \midrule
CIFAR-10   & 45.2 & 48.8 & 47.5 & 49.0 & 49.2 \\ 
CIFAR-100  & 23.4 & 33.4 & 48.7 & 50.9 & 50.5 \\ 
\bottomrule
\end{tabular}
\label{tab:setting_steps_comparison}
\end{table}

\cut{
\section{Conclusion}
In this work, we explore the generalization ability of condensed datasets produced by training trajectory-matching-based algorithms, optimizing both sharpness and trajectory distance. We propose \ourModel{} (\ouracronym{}) to reduce the computational cost of long-horizon inner-loops and mini-max optimization via hypergradient approximation strategies, which are theoretically grounded, practically effective. Our approach improves generalization across in- and out-of-domain tasks, including cross-architecture and continual learning, and can be successfully deployed to the challenging ImageNet-1K task with clear improvement on both generalizations performance and computational overhead compared with other methods. Additionally, \ouracronym{} serves as a "plug-and-play" model for other methods especially the curriculum learning-based method resulting in further improvement. Future research could explore advanced gradient estimation techniques, such as implicit gradient, to enhance computational efficiency and reduce approximation error.
}
\cut{
\section*{Impact Statement}
This paper aims to contribute to the advancement of Machine Learning. While our work may have various societal implications, none require specific emphasis in this context.
}

\section{Conclusions, Limitations and Future Works \label{conclusion}}
In this work, we explore the generalization ability of condensed datasets produced by training trajectory-matching-based algorithms via jointly optimising the sharpness and the distance between real and synthetic trajectories. We propose \ourModel{} (\ouracronym{}) to reduce the computational cost caused by the long horizon inner-loop and the mini-max optimization for the sharpness minimization through the proposed hypergradient approximation strategies. Those strategies have clear theoretical motivation, limited error in practice, and a framework flexible enough to adapt to other sharpness-aware based algorithms. The improvement of the generalization is observed in a variety of in- and out-of-domain tasks such as cross-architecture and cross-task (continual learning) with a comprehensive analysis of the algorithm's sharpness properties on the training dynamics. 

Despite the superior performance of \ouracronym{}, we observed that the proposed algorithm serves as a "plug-and-play" model for other dataset condensation methods and more broadly, for various bilevel optimization applications, such as loss function learning and optimizer learning. However, these possibilities are not explored in this work and we leave them to the future work. Moreover, beyond focusing on reusing the trajectory to enhance training efficiency in reaching flat regions, future research could be in advanced gradient estimation directions, such as implicit gradients, showing promise for managing long-horizon inner-loops and avoiding second-order unrolling. This could eliminate the entire second trajectory resulting in higher computational efficiency and less approximation error.  

\newpage
\bibliography{reference}
\bibliographystyle{icml2025}

\newpage
\appendix
\onecolumn
\section{Appendix}
\subsection{Proof for Theorem~\ref{theorem:traj_difference} \label{proof:traj_difference}}

\trajDifference*
\begin{proof}
Let:
\begin{align*}
\hat{\theta}_{\tau}&= \theta_{0} - \alpha \sum^{\tau}_{i} \nabla \mathcal{L}(\phi+\epsilon, \hat{\theta}_i) \\
\theta_{\tau} & = \theta_{0} - \alpha \sum^{\tau}_{i} \nabla \mathcal{L}(\phi, \theta_i) 
\end{align*}
then after N step iterations, the difference between $\theta_N$ and $\hat{\theta}_N$ is 
\begin{align*} 
\left\|\Delta \theta_{\tau} \right\| = \left\|\hat{\theta}_{\tau} - \theta_{\tau}\right\|  
& = \left\|-\alpha \sum^{\tau}_i (\nabla \mathcal{L}(\phi+ \epsilon, \hat{\theta}_i) - \nabla \mathcal{L}(\phi, \theta_i) ) \right\| \\ 
& = \alpha \left\| \sum^{\tau}_i (\nabla \mathcal{L}(\phi+ \epsilon, \hat{\theta}_i) - \nabla \mathcal{L}(\phi, \theta_i) ) \right\|
\end{align*}
We compute the gradient difference: 
\begin{align*}
    & ||\nabla \mathcal{L}(\phi + \epsilon, \hat{\theta}_i) - \nabla \mathcal{L}(\phi, \theta_i)|| \\
    & \approx ||\nabla \mathcal{L}(\phi, \hat{\theta}_i) + \nabla_{\phi} \nabla_{\theta} \mathcal{L}(\phi, \hat{\theta}_i) \cdot \epsilon - \nabla \mathcal{L}(\phi, \theta_i)|| \\
    & \leq ||\nabla \mathcal{L}(\phi, \hat{\theta}_i) - \nabla \mathcal{L}(\phi, \theta_i)|| + || \nabla_{\phi} \nabla_{\theta} \mathcal{L}(\phi, \hat{\theta}_i) \cdot \epsilon || \\
    & \leq 2\sigma + || \nabla_{\phi} \nabla_{\theta} \mathcal{L}(\phi, \hat{\theta}_i)|| || \epsilon ||
\end{align*}
With $ \nabla_{\phi} \nabla_{\theta} \mathcal{L}(\phi, \hat{\theta}_i)$ is $\beta$ smooth and $|| \epsilon|| = \rho$ : 
\begin{align*}
    ||\nabla \mathcal{L}(\phi + \epsilon, \hat{\theta}_i) - \nabla \mathcal{L}(\phi, \theta_i)||_2 \leq 2\sigma + \beta \rho
\end{align*}
Then:
\begin{align*}
\left\| \Delta \theta_{\tau} \right\| \leq \alpha \tau (2\sigma + \beta \rho) 
\end{align*}
\end{proof}

\subsection{Proof of Proposition~\ref{prop:gradient_similiarity}} \label{proof:gradient_similiarity}
\gradientSimiliarity*
\begin{proof}
        Let 
        \begin{align*}
             A_{i+1} = \frac{\partial \theta_{i+1}}{\partial \theta_i},  B_{i+1} = \frac{\partial \theta_{i+1}}{\partial \phi}
        \end{align*}
        then
        \begin{align*}
        \frac{\partial F(\phi)}{\partial \phi} = \frac{\partial \mathcal{L}(\theta(\phi))}{\partial \phi} + \sum^{N}_{i = 0} B_{i}A_{i+1} \cdots A_{N} \frac{\partial \mathcal{L}(\theta(\phi))}{\partial \theta_N (\phi)}
        \end{align*}
	Let $e_{\iota} = \frac{\partial F(\phi)}{\partial \phi} - \frac{\partial F_{\iota}(\phi)}{\partial \phi}$, 
	\begin{align*}
	e_{\iota} =  \left( \sum_{i=0}^{\iota}  B_{i} A_{i+1} \cdots A_{\iota}\right) A_{\iota+1} \cdots A_{N}  \frac{\partial \mathcal{L}(\theta(\phi))}{\partial \theta_N (\phi)}
	\end{align*}
    Given $\mathcal{L}_{CE}$ is locally $J$-strongly convex with respect to $\theta$ in the neighborhood of $\{\theta_{\iota+1}, \dots, \theta_N\}$,
	\begin{align*}
	\norm{e_{\iota}} &\leq \Bigg\| \sum_{i=0}^{\iota}  B_{i} A_{i+1} \cdots A_{\iota}\Bigg\| \Bigg\| A_{\iota+1} \cdots A_{N} \frac{\partial \mathcal{L}(\theta(\phi))}{\partial \theta_N (\phi)}\Bigg\| \\
	&\leq (1 - \alpha J)^{N-\iota +1} \left\| \frac{\partial \mathcal{L}(\theta(\phi))}{\partial \theta_{N}(\phi)} \right\|   \Bigg\| \sum_{i=0}^{\iota}  B_{i} A_{i+1} \cdots A_{\iota}\Bigg\| 
	\end{align*}
 
    In the worst case, when $\mathcal{L}_{CE}$ is $K$-smooth but nonconvex, then if the smallest eigenvalue of $\frac{\partial^2 \mathcal{L}_{CE}(\theta, \phi)}{\partial \theta \,\, \partial \theta}$ is $-K$, then $\norm{A_i} = 1 + \alpha K \leq 2$ for $i = 0,\dots, \iota$. 
\end{proof}
\subsection{Pytorch Based Pseudocode for Truncated Unrolling Hypergradient \label{pytorch_psed}}
\begin{algorithm}[h]
\small
\caption{Trucated hypergradient computation}
\begin{algorithmic}

\STATE \texttt{stop gradient}: 
    \FOR{$i= 1, \ldots, \iota$}
        \STATE $\theta_{i} = \theta_{i-1} - \alpha * \text{ torch.grad}(\mathcal{L}_{CE}(\theta, \phi), \theta)$ 
    \ENDFOR
        
    \STATE \texttt{with gradient}: 
    \FOR{$i= 1, \ldots, N-\iota$}
        \STATE $\theta_{i} = \theta_{i-1} - \alpha * \text{torch.grad}(\mathcal{L}_{CE}(\theta, \phi), \theta, \text{ retain\_graph} = \text{ True}, \text{create\_graph} = \text{ True})$
        
    \ENDFOR
    \STATE \textbf{Return:} $\theta_N(\phi)$
  \end{algorithmic}
\end{algorithm}

\subsection{The Derivation of Learning Rate Learning with First Order Derivative \label{lr_fod}}
In this section, we provide the derivation of the hypergradient calculation for learning rate $\alpha$ and the visual comparison of the learning rate learning dynamics generated by the first-order and second-order methods. Given the outer-loop objective, $\mathcal{L}(\theta(\phi))$, and the inner-loop object $\mathcal{L}_{CE}(\theta_i, \phi)$ with $N$ iteration unrolling, the computation can be dedicated by:  
\begin{align*}
\frac{\partial \mathcal{L}(\theta_N(\phi))}{\partial \alpha} & =  \frac{\partial \mathcal{L}(\theta_N(\phi))}{\partial \theta_{N}} \cdot  \frac{\partial (\theta_N, \phi)}{\partial \alpha} \\
& =  \frac{\partial \mathcal{L}(\theta_N(\phi))}{\partial \theta_{N}} \cdot  \frac{\partial \Xi(\theta_{N-1}, \phi)}{\partial \alpha} \\
& =  \frac{\partial \mathcal{L}(\theta_N(\phi))}{\partial \theta_{N}} \cdot  \frac{\partial}{\partial \alpha}\Bigg( \theta_{N-1} - \alpha \frac{\partial \mathcal{L}_{CE}(\theta_{N-1}, \phi)}{\partial \theta_{N-1}} \Bigg) \\
& =  \frac{\partial \mathcal{L}(\theta_N(\phi))}{\partial \theta_{N}} \cdot \Bigg(  \frac{\partial \theta_{N-1}}{\partial \alpha}  - \frac{\partial \mathcal{L}_{CE}(\theta_{N-1}, \phi)}{\partial \theta_{N-1}} \Bigg) \\
& \text{we treat } \frac{\partial \mathcal{L}_{CE}(\theta_{N-1}, \phi)}{\partial \theta_{N-1}} \text{ as a constant w.r.t. } \alpha \\
& =  \frac{\partial \mathcal{L}(\theta_N(\phi))}{\partial \theta_{N}} \cdot \Bigg(  \frac{\partial }{\partial \alpha} \Xi(\theta_{N-2}, \phi) - \frac{\partial \mathcal{L}_{CE}(\theta_{N-1}, \phi)}{\partial \theta_{N-1}} \Bigg) \\
& =  \frac{\partial \mathcal{L}(\theta_N(\phi))}{\partial \theta_{N}} \cdot \Bigg( - \sum^{N-1}_{i = 0} \frac{\partial \mathcal{L}_{CE}(\theta_i, \phi)}{\partial \, \theta_{i}}\Bigg)  \\
\end{align*}

 We compare the learning rate learning dynamic produced by first-order (our method) and second-order derivatives, demonstrating limited differences between those two methods. The visualisation comparison of the learning rate learning dynamic produced by the first and second-order derivative is illustrated in Fig.~\ref{fig:lr_lr}. As can be noticed, two inner-loop trajectories in the sharpness aware setting are capable of this Eq.~\ref{eq:lr_learning}. We chose the first in our experiments due to the implementation simplicity without causing any significant performance differences.  
\begin{figure}[!h]
    \centering
    \includegraphics[width=0.5\linewidth]{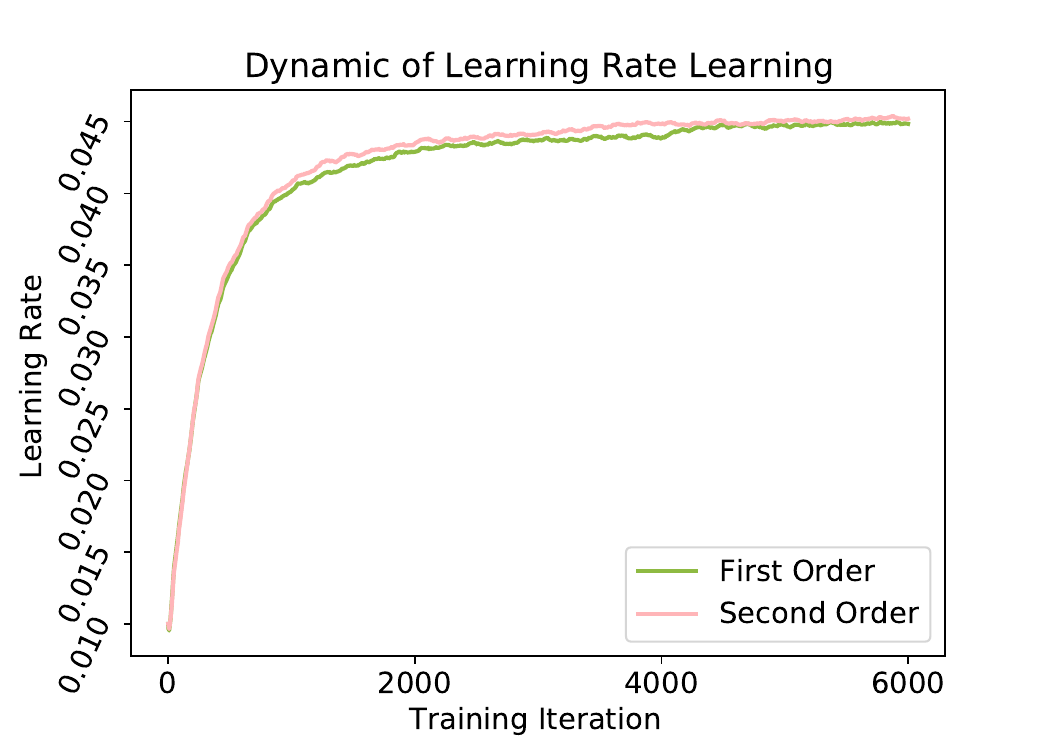}
    \caption{The comparison of the learning dynamic of learning rate learning with first and second order differentiation when condensing on the Cifar100-10IPC setting.}
    \label{fig:lr_lr}
\end{figure}

\subsection{Computational Resource\label{sec:computational_resource}}
We conduct all our experiments on two TESLA V100-32GB GPUs with Intel(R) Xeon(R) W-2245 CPU @ 3.90GHz and one A100-40GB GPU with Intel(R) Xeon(R) Gold 5118 CPU @ 2.30GHz which are on different servers. Thus, we cannot run the full batch of synthetic dataset learning as the same as other trajectory matching-based methods when the inner-loop trajectories contain many unrolling iterations. Those cases include Cifar100-10IPC, Cifar100-50IPC, and Tiny ImageNet 1IPC. In our case, stochastic gradient descent with mini-batch is utilised in the outer-loop instead. 
\cut{
\subsection{Computational Cost Comparsion \label{computational_cost_appendix}}
We computed and recorded the memory and time costs when running SATM and then compared them with MTT and TESLA following TESLA's experimental protocol. The results were primarily measured on a single NVIDIA A6000 GPU, except for MTT on ImageNet-1K~\citep{imagenet}, which required two A6000 GPUs.

In our experiments, at most only one-third of the inner-loop is retained to compute the hypergradients for sharpness approximation and synthetic dataset optimization.  Given the result in Table~\ref{tab:memory_comparison}, our strategy significantly reduces memory consumption compared to MTT, enabling the dataset to be trained on a single A6000 GPU. 

Regarding time cost illustrated in Table~\ref{tab:time_comparison}, SATM consistently outperforms the two inner-loop-based algorithms, TESLA. In the one-third inner-loop case, SATM even consumes less time than MTT which requires retaining a full single inner-loop.

\begin{table}[t]
\centering
\small
\begin{tabular}{c|cc| cc}
\toprule
Model         & \multicolumn{2}{c}{Memory}          & \multicolumn{2}{c}{Runtime} \\ 
              &   CIFAR-100      &   ImageNet-1K    &      CIFAR-100     &     ImageNet-1K  \\ \midrule
MTT           & 17.1$\pm$0.1 & 79.9$\pm$0.1 & 12.1$\pm$0.6 & 45.9$\pm$0.5 \\
TESLA         & 3.6$\pm$0.1 & 13.9$\pm$0.1 & 15.3$\pm$0.5 & 47.4$\pm$0.7 \\
SATM          & 5.7$\pm$0.1 & 26.6$\pm$0.1 & 12.0$\pm$0.5 & 45.4$\pm$0.4 \\
\bottomrule
\end{tabular}
\caption{GPU memory and runtime comparison among MTT, TESLA and SATM on CIFAR100 and ImageNet-1K with results measured with a batch size of 100 and 50 inner-loop steps. }
\label{tab:time_comparison}
\end{table}
}

\cut{
\begin{figure*}[!h]
    \centering
    \includegraphics[width=0.6\linewidth]{figures2/runtime and memory comparison.pdf}
\caption{GPU memory and runtime comparison among MTT, TESLA and SATM (N/3) on CIFAR100 and ImageNet-1K with results measured with a batch size of 100 and 50 inner-loop steps. }
\end{figure*}
}
\cut{
To further justify the memory efficiency of SATM, we challenge the ImageNet-1K setting following the training and evaluation protocol from TESLA. By truncating the inner-loop computational graph hold for hypergradient computation, SATM is executable on the heavy memory setting with results given in Table~\ref{tab:TESLA_satm_comparison}.
\begin{table}[h!]
\centering
\begin{tabular}{cc| cc}
\toprule
Dataset      & IPC & TESLA         & SATM             \\ \midrule
ImageNet-1K           & 1           & 7.7$\pm$0.2             & \textbf{8.9}$\pm$0.4      \\ 
                      & 2           & 10.5$\pm$0.3            & \textbf{11.4}$\pm$0.2     \\ 
                      & 10          & 17.8$\pm$1.3            & \textbf{19.2}$\pm$0.9     \\ 
                      & 50          & 27.9$\pm$1.1            & \textbf{29.2}$\pm$1.1     \\ 
\bottomrule
\end{tabular}
\caption{Comparison of TESLA and SATM across different IPCs on ImageNet-1K.}
\label{tab:TESLA_satm_comparison}
\end{table}
}

\subsection{Flat Inner-loop Study}
SATM is developed based on MTT without incorporating the components introduced in FTD~\citep{ftd}, particularly the expert trajectories generated by sharpness-aware optimizers such as GSAM. However, understanding whether SATM can be compatible with advanced expert trajectories is desirable to study. Therefore, we follow the expert trajectory generation protocol and execute SATM on the flat expert trajectories with the results in Table~\ref{tab:comparison_satm-fi}. It can be observed that the inclusion of a flat inner-loop leads to clear improvements in SATM-FI compared to both standard SATM and FTD. Furthermore, the authors of FTD noted the limited performance contribution of EMA, which was originally intended to guide the synthetic dataset toward convergence on a flat loss landscape. SATM addresses this limitation and effectively demonstrates the benefits of leveraging flatness for improved generalization.
\begin{table}[h!]
\centering
\small
\begin{tabular}{lc|cccc}
\toprule
                 & IPC   & MTT          & FTD              & SATM             & SATM-FI              \\ \midrule
                 & 1     & 46.2$\pm$0.8 & 46.8$\pm$0.3     & \textbf{49.0}$\pm$0.3     & 48.7$\pm$0.4 \\ 
CIFAR-10         & 10    & 65.4$\pm$0.7 & 66.6$\pm$0.3     & 67.1$\pm$0.4     & \textbf{67.9}$\pm$0.3 \\ 
                 & 50    & 71.6$\pm$0.2 & 73.8$\pm$0.2     & 73.9$\pm$0.2     & \textbf{74.2}$\pm$0.4 \\ \midrule
                 & 1     & 24.3$\pm$0.3 & 25.2$\pm$0.2     & 26.1$\pm$0.4     & \textbf{26.6}$\pm$0.5 \\ 
CIFAR-100        & 10    & 39.7$\pm$0.4 & 43.4$\pm$0.3     & 43.1$\pm$0.5     & \textbf{43.9}$\pm$0.7 \\ 
                 & 50    & 47.7$\pm$0.2 & 50.7$\pm$0.3     & 50.9$\pm$0.5     & \textbf{51.4}$\pm$0.5 \\ \midrule
Tiny-ImageNet    & 1     & 8.8$\pm$0.3  & 10.4$\pm$0.3     & 10.9$\pm$0.2     & \textbf{11.7}$\pm$0.4 \\ 
                 & 10    & 23.2$\pm$0.1 & 24.5$\pm$0.2     & 25.4$\pm$0.4     & \textbf{25.6}$\pm$0.6 \\ 
\bottomrule
\end{tabular}
\caption{Accuracy (\%) Comparison of MTT, FTD, SATM, and SATM-FI across different datasets and configurations.}
\label{tab:comparison_satm-fi}
\end{table}

\subsection{Compatibility with Advanced Sharpness-Aware optimizers.} We study the compatibility of the proposed hypergradient approximation method on other sharpness minimization-based methods including EMA, SAM~\citep{sam}, GSAM~\citep{gsam}, ASAM~\citep{asam} and Vasso~\citep{vasso} with our loss landscape smoothing mechanism removed. For a fair comparison, the hyperparameters of each method are properly tuned for the adaption to all the tasks including Cifar100 with 1 IPC and Tiny ImageNet with 3 IPC. We repeat each method 5 times and report the mean and variance in Table~\ref{tab:sam_family}. The results imply that all the sharpness methods consistently improve MTT~\citep{mtt}, which justifies the benefit of sharpness minimization. However, the competitors all fail to defeat our method due to the failure to accurately compute the sharpness proxy. Moreover, EMA, equivalent to FTD without Sharpness-aware minimizers to generate expert trajectories, gains minimal improvement. 
\begin{table*}[!h]
    \centering
    \small
    \label{tab:sam_family}
    \resizebox{0.85\textwidth}{!}{
    \begin{tabular}{c |c c c c c c | c }
    \toprule
    Dataset (IPC) & MTT & EMA &  SAM  & GSAM      & ASAM    &  Vasso  &  \ouracronym{}  \\
    \midrule
    Cifar100 (1)            & $24.3{\pm0.4}$   & $24.7{\pm0.2}$ & $25.7{\pm0.3}$ & $25.9{\pm0.3}$ & $25.7\pm{0.3}$ & $25.9{\pm0.2}$ & $\textbf{26.1}{\pm0.3}$ \\
    Tiny ImageNet (3)       & $10.5{\pm0.3}$   & $10.9{\pm0.3}$  & $12.3{\pm0.2}$  & $13.1{\pm0.2}$  & $12.8\pm{0.4}$  & $12.2{\pm0.2}$ &  $\textbf{13.6}{\pm0.2}$ \\
    \bottomrule
    \end{tabular}
    }
    \caption{Test Accuracy (\%) Comparison with the advanced sharpness aware minimization methods including EMA, SAM, GSAM, ASAM and Vasso with the same expert trajectories as MTT. }
\end{table*}
\cut{
\subsection{More related work and comparison with Recent Method}
A recent method, RDED~\citep{rded}, introduces new perspectives to the dataset distillation field by constructing synthetic images from original image crops and labelling them with a pre-trained model. In comparison, our work falls within the training trajectory matching area and focuses on efficient bilevel optimization with a long inner-loop with the goal of enhancing the generalization ability of synthetic data by developing an efficient, sharpness-aware optimizer for bilevel optimization. 
DATM~\citep{datm} utilizes the difficulty of training trajectories to implement a curriculum learning-based dataset condensation protocol. While this approach is relevant, it is somewhat distinct from research focused on optimization efficiency and generalization, such as TESLA, FTD, and SATM, which prioritize optimization efficiency through gradient approximation. Additionally, from an implementation perspective, DATM feeds expert trajectories in an easy-to-hard sequence directly into FTD. In contrast, our work focuses on the flatness of the loss landscape of the learning dataset from a bilevel optimization perspective, rather than emphasizing pure performance comparisons. Nevertheless, we believe our method is compatible with DATM. To demonstrate this, we conducted experiments combining DATM's easy-to-hard training protocol with SATM, yielding the following results in Table~\ref{tab:comparison_ipc_datm_satm_appendix}.
\begin{table}[h!]
\centering
\small
\begin{tabular}{lc|cccc}
\toprule
                & IPC & MTT & FTD & DATM & SATM-DA \\ \midrule
                & 1   & $46.2\pm0.8$ & $46.8\pm0.3$ & $46.9\pm0.5$ & $\mathbf{48.6}\pm0.4$ \\ 
CIFAR-10        & 10  & $65.4\pm0.7$ & $66.6\pm0.3$ & $66.8\pm0.2$ & $\mathbf{68.1}\pm0.3$ \\ 
                & 50  & $71.6\pm0.2$ & $73.8\pm0.2$ & $76.1\pm0.3$ & $\mathbf{76.4}\pm0.6$ \\ \midrule
                & 1   & $24.3\pm0.3$ & $25.2\pm0.2$ & $27.9\pm0.2$ & $\mathbf{28.2}\pm0.8$ \\ 
CIFAR-100       & 10  & $39.7\pm0.4$ & $43.4\pm0.3$ & $47.2\pm0.4$ & $\mathbf{48.3}\pm0.4$ \\ 
                & 50  & $47.7\pm0.2$ & $50.7\pm0.3$ & $55.0\pm0.2$ & $\mathbf{55.7}\pm0.3$ \\ \midrule
Tiny-ImageNet   & 1   & $8.8\pm0.3$  & $10.4\pm0.3$ & $\mathbf{17.1}\pm0.3$ & $16.4\pm0.4$ \\ 
                & 10  & $23.2\pm0.1$ & $24.5\pm0.2$ & $31.1\pm0.3$ & $\mathbf{32.3}\pm0.6$ \\ 
\bottomrule
\end{tabular}
\caption{Accuracy (\%) Comparison of MTT, FTD, DATM, and SATM-DA across different IPCs, datasets and configurations.}
\label{tab:comparison_ipc_datm_satm_appendix}
\end{table}
}

\subsection{Experiment Setting Details \label{experiment_setting_details}}
We conduct experiments on four main image datasets, Cifar10~\citep{cifar10}, Cifar100~\citep{cifar10},  TinyImageNet~\citep{tinyimg} and ImageNet~\citep{imagenet}. Cifar10 categorises 50,000 images with the size $32\times32$ into 10 classes while Cifar100 further categorises each of those 10 classes into 10 fine-grained subcategories. TinyImageNet comprises 100,000 images distributed across 200 categories, each category consisting of 500 images resized to dimensions of $64\times64$. We further evaluate \ouracronym{} on the subset of ImageNet, namely ImageNette, Image Woof, ImageFruit and ImageMeow with each set containing 10 different categories of $128\times128$ images and the whole ImageNet following the protocol from TESLA~\citep{TESLA}. 

We evaluate our methods on four main image datasets, Cifar10~\citep{cifar10}, Cifar100~\citep{cifar10}, TinyImageNet~\citep{tinyimg} and ImageNet-1K~\citep{imagenet}. The expert trajectories for Cifar10 and Cifar100 are trained with 3-layer ConvNet and collected after each epoch with the initialisation, and those for TinyImageNet and ImageNet are trained with 4-layer and 5-layer ConvNet~\cite{gidaris2018dynamic} respectively. In the in-domain setting, the synthetic datasets are learned and evaluated on the same architectures while in the out-of-domain settings, the learned synthetic datasets are deployed to train different architectures, such as AlexNet~\citep{Alexnet}, VGG11~\citep{vgg} and ResNet18~\citep{resnet}, which is novel to the synthetic datasets. The trained neural networks are evaluated on the real test sets for generalization ability comparison of the synthetic datasets. 

\subsection{Hyperparameters and Experiment Details~\label{sec:hyperparameters}}

The hyperparameters used for condensing datasets in all the settings are given in Tab~\ref{tab:hyperparameters} with ConvNet~\citep{gidaris2018dynamic} applied to construct the training trajectories. 

\begin{table}[h]
\centering
\small
\resizebox{1.0\linewidth}{!}{
\begin{tabular}{ccc|ccccccc}
\toprule
Dataset & Model & IPC & \makecell{Synthetic \\ Steps \\ ($N$)} & \makecell{Expert \\ Epochs \\($M$)}  & \makecell{Max Start \\ Epoch \\ ($T$)}  & \makecell{Synthetic \\ Batch Size \\ } & ZCA & \makecell{Learning \\ Rate \\(Images)} &  \makecell{Learning \\ Rate \\(Step size)}\\ \midrule
\multirow{3}{*}{CIFAR-10}  & \multirow{4}{*}{ConvNetD3}    &  1  & 50      & 2  & 2   & -  & Y & 1000 &1$\times10^{-6}$\\
                               &                           &  3  & 50      & 2  & 2   & -  & Y & 100  &1$\times10^{-5}$\\
                               &                           & 10  & 30      & 2  & 20  & -  & Y & 50   &1$\times10^{-5}$ \\
                               &                           & 50  & 30      & 2  & 40  & -  & Y & 100  &1$\times10^{-5}$ \\ \midrule
                                
\multirow{3}{*}{CIFAR-100} & \multirow{3}{*}{ConvNetD3}     &  1  & 40     & 3  & 20  & -      & Y  & 500  & 1$\times10^{-5}$\\ 
                              &                             &  3  & 45     & 3  & 20  & -      & Y  & 1000 & 5$\times10^{-5}$\\  
                              &                             & 10  & 20     & 2  & 20  & 500    & Y  & 1000 & 1$\times10^{-5}$\\  
                              &                             & 50  & 80     & 2  & 40  & 500    & Y  & 1000 & 1$\times10^{-5}$  \\  \midrule

\multirow{3}{*}{Tiny ImageNet} & \multirow{3}{*}{ConvNetD4} &  1  & 30      &  2  &  10  &  200 & Y & 1000  & 1$\times10^{-4}$\\ 
                               &                            &  3  & 30      &  2  &  15  &  200 & Y & 1000  & 1$\times10^{-4}$\\  
                               &                            & 10  & 20      &  2  &  40  &  200 & Y & 10000 & 1$\times10^{-4}$ \\  
\bottomrule
\end{tabular}
}
\caption{Hyper-parameters used for our \ouracronym{}. A synthetic batch size of ``-'' represents that a full batch set is used in each outer-loop iteration. ConvNetD3 and ConvNet4D denote the 3-layer and 4-layer ConvNet~\citep{gidaris2018dynamic} respectively. In all the settings, ZCA whitening~\citep{nguyen2021dataset, nguyen2022dataset} is applied. }
\label{tab:hyperparameters}
\end{table}

\subsection{Illustration for the Synthetic Images}
We visualise the learned synthetic datasets on Cifar10, Cifar100 and Tiny ImageNet in this section.
\begin{figure*}[!h]
    \centering
    \includegraphics[width=0.75\linewidth]{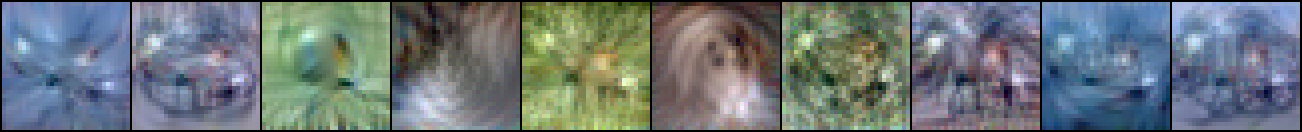}
\caption{Cifar10 with 1IPC}
\end{figure*}

\begin{figure*}[h!]
    \centering
    \includegraphics[width=0.7\linewidth]{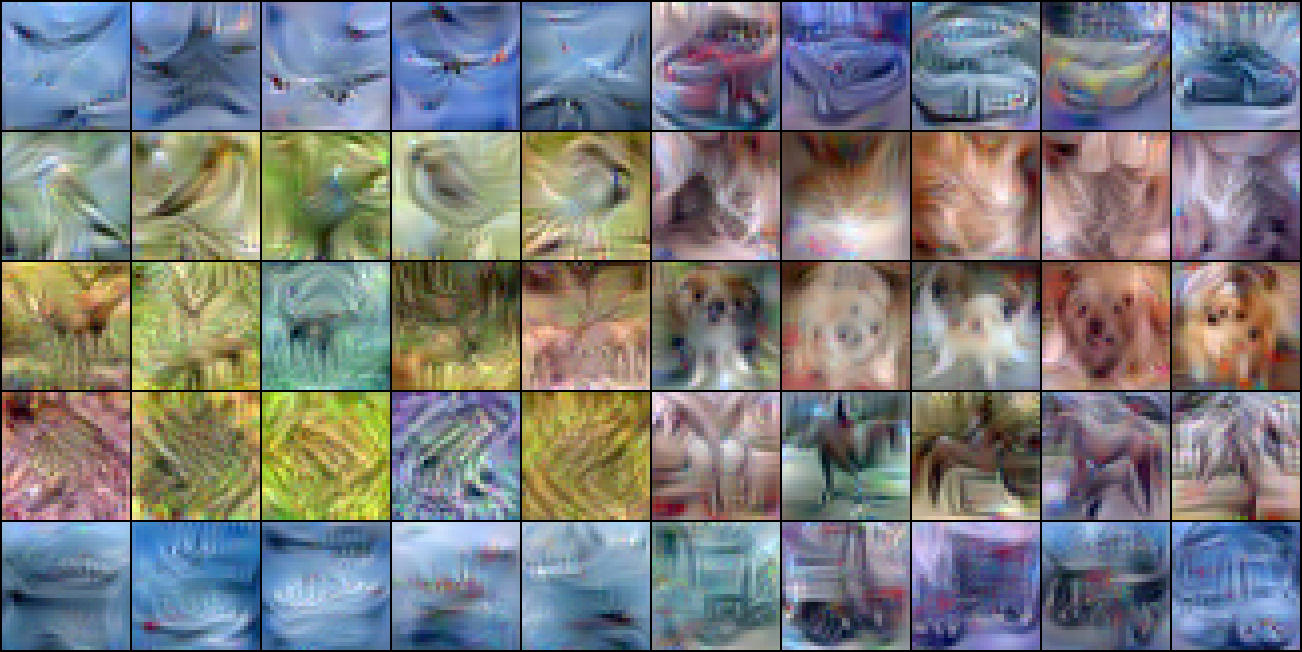}
\caption{Cifar10 with 3IPC}
\end{figure*}
\begin{figure*}[h!]
    \centering
    \includegraphics[width=0.7\linewidth]{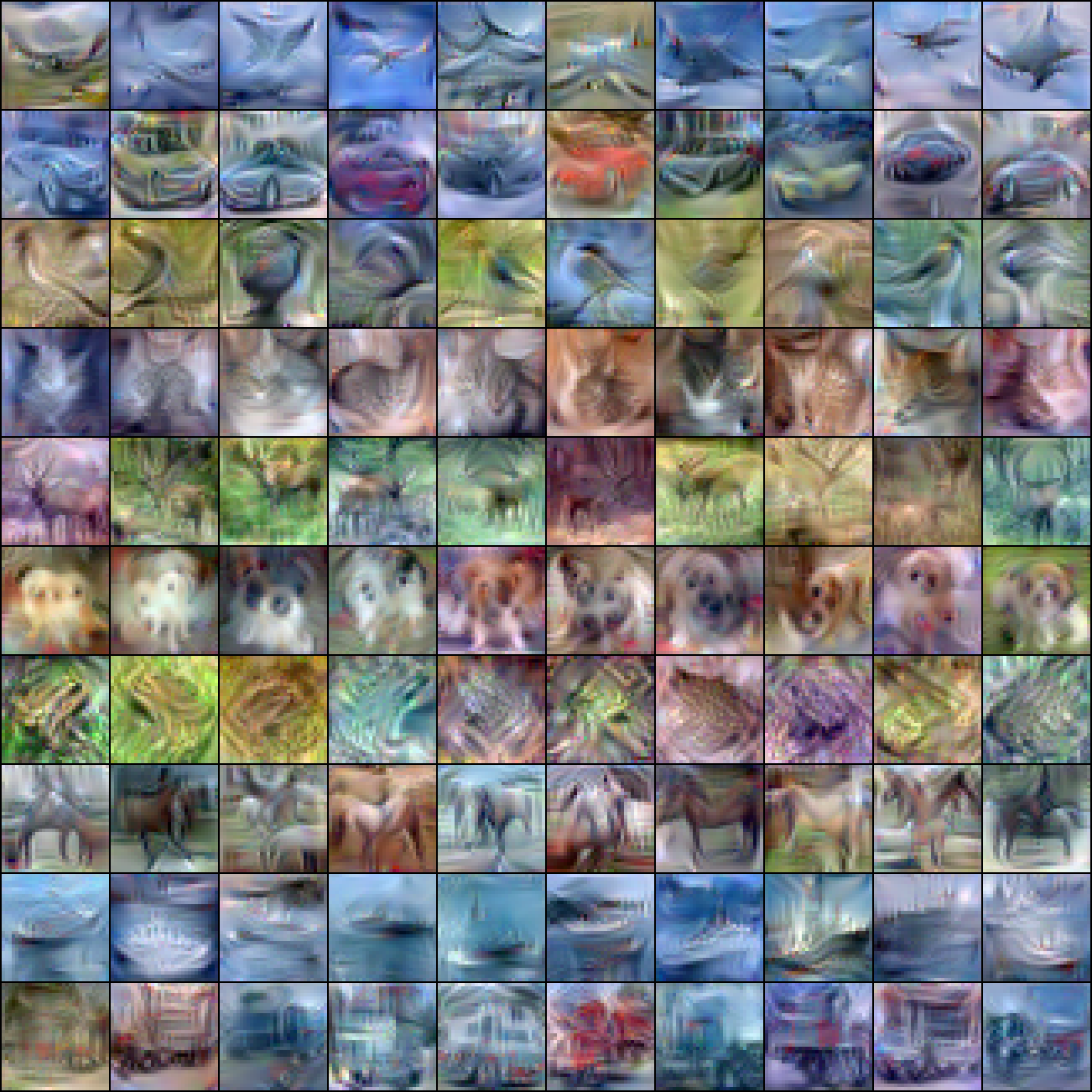}
\caption{Cifar10 with 10IPC}
\end{figure*}
\cut{
\begin{figure*}[h!]
    \centering
    \includegraphics[width=0.6\linewidth]{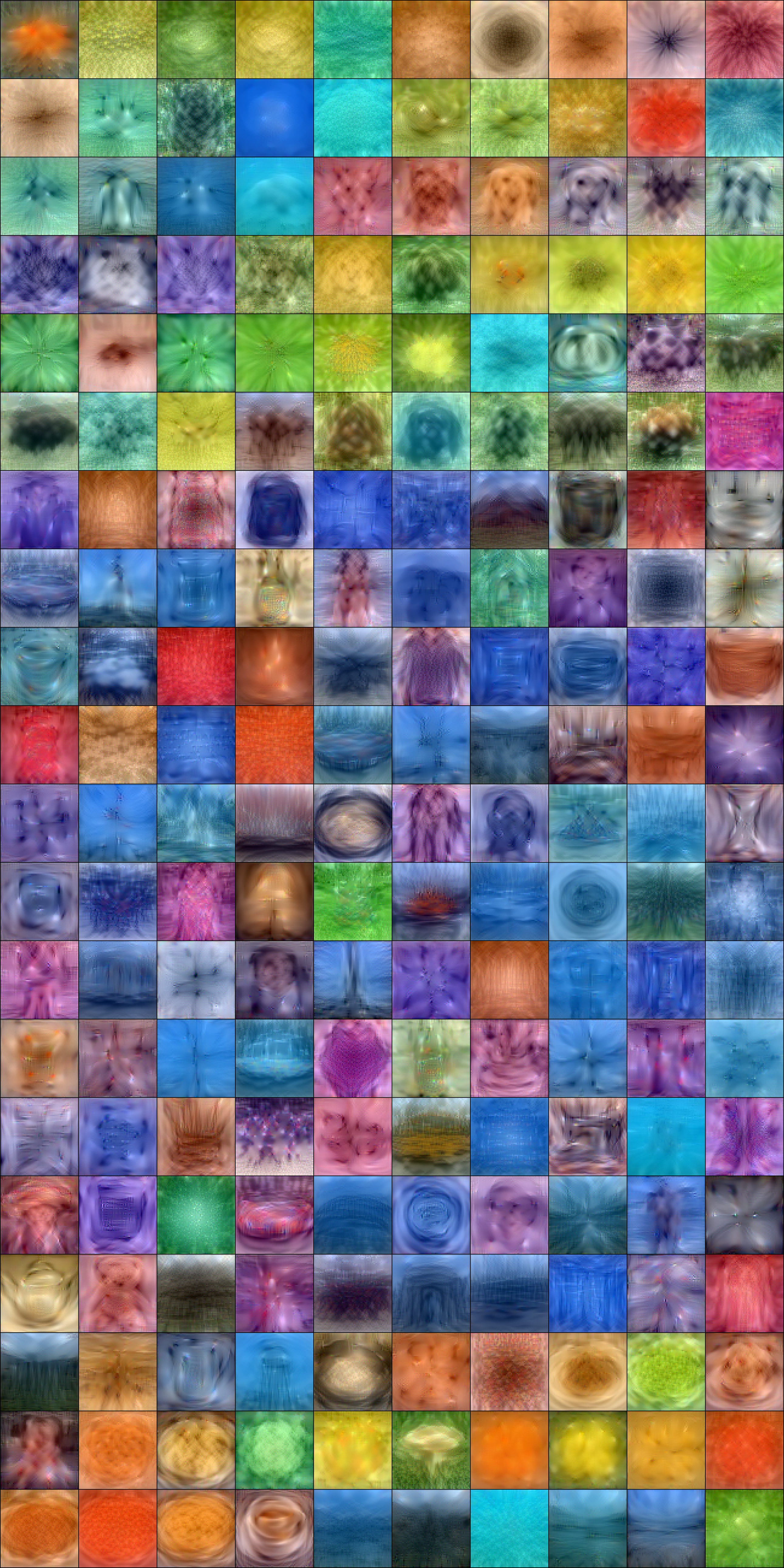}
\caption{Cifar100 with 1IPC}
\end{figure*}
}
\begin{figure*}[h!]
    \centering
    \includegraphics[width=0.6\linewidth]{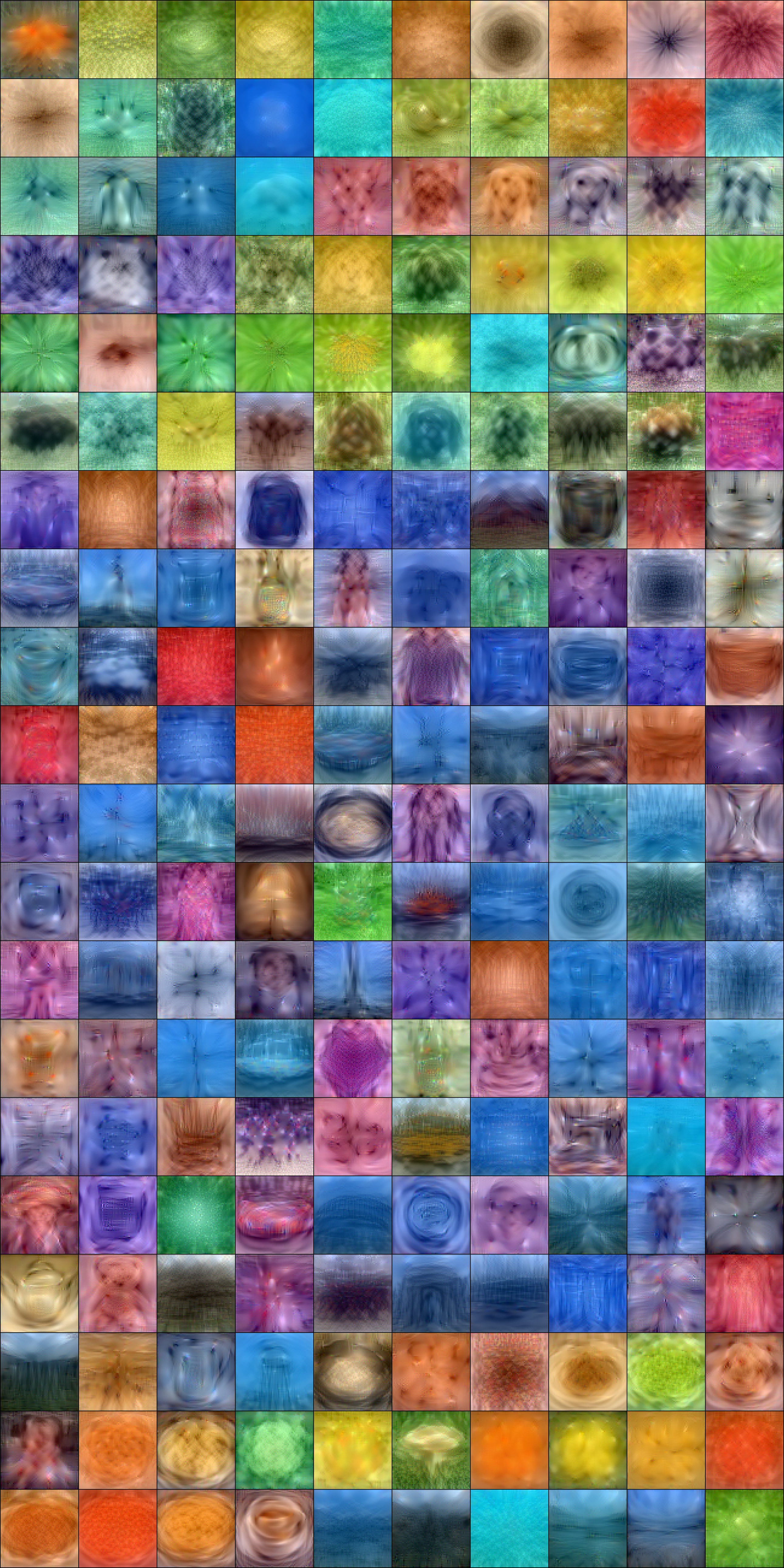}
\caption{Cifar100 with 1IPC}
\end{figure*}
\end{document}